%% file: 00_main.tex
\documentclass[twoside,11pt]{article}

\usepackage{jmlr2e}
\usepackage{booktabs}

\usepackage{lastpage}

\jmlrheading{24}{2023}{1-\pageref{LastPage}}{1/23; Revised 12/23}{12/23}{23-0041}{Danny Wood, Tingting Mu, Andrew M. Webb, Henry W. J. Reeve, Mikel Luj\'an, and Gavin Brown}

\ShortHeadings{A Unified Theory of Diversity in Ensemble Learning}{Wood, Mu, Webb, Reeve, Luj\'an, Brown}

\author{\name Danny Wood \email danny.wood@manchester.ac.uk$^\dagger$ \\
\name Tingting Mu \email tingting.mu@manchester.ac.uk$^\dagger$ \\
\name Andrew M. Webb \email andrew.webb@manchester.ac.uk$^\dagger$ \\
\name Henry W. J. Reeve \email henry.reeve@bristol.ac.uk$^\ast$ \\
\name Mikel Luj{\'a}n \email mikel.lujan@manchester.ac.uk$^\dagger$ \\
\name Gavin Brown \email gavin.brown@manchester.ac.uk$^\dagger$\\
\addr $^\dagger$ Department of Computer Science, University of Manchester, UK\\
\addr $\ast$ School of Mathematics, University of Bristol, UK\\
}

\usepackage[export]{adjustbox}

\usepackage{pdflscape}

\usepackage{array}

\usepackage{algorithm2e}
\usepackage{bbm}
\usepackage{nicefrac}
\usepackage{mathtools}
\usepackage{enumitem}
\usepackage{graphicx}
\usepackage{dsfont}
\usepackage{multicol}

\usepackage{xcolor,colortbl}
\usepackage{subfigure}
\usepackage{comment}
\usepackage[usestackEOL]{stackengine}

\usepackage{cleveref}
\usepackage{pgfplots} 
\pgfplotsset{compat=1.16} 

\usepackage{graphbox}

\usepackage{thmtools}
\usepackage{thm-restate}

\usepackage{threeparttable}
\usepackage{amssymb}
\usepackage{pifont}
\usepackage[normalem]{ulem}

\usepackage{booktabs}

\usepackage{wrapfig}

\usepackage{float}
\restylefloat{table}

\usepackage{accents}
\usepackage{amsmath}

\usepackage{amsmath,amssymb,braket,tikz}
\usetikzlibrary{arrows,shapes,trees,positioning, automata, intersections, calc, arrows.meta}  

\input{0_definitions.tex}

\pgfplotsset{compat=1.16}

\begin{document}

\title{A Unified Theory of Diversity in Ensemble Learning}

\editor{Boaz Nadler}

\maketitle

\input{0_abstract.tex}


\input{01_Introduction}

\input{01_Notation}
\input{02_ProblemStatement}

\input{03_VeryShortIntro}

\input{04_Unified}

\input{05_BregmanDiversity}
\input{06_Discussion}

\input{07_Conclusions}

\newpage

\input{99_A_AdditionalExperiments.tex}

\input{99_B_Exposition}

\input{99_C_Section4proofs}

\input{99_D_Section5proofs}

\bibliography{biblio}

\end{document}

%% file: 0_definitions.tex
\DeclarePairedDelimiterX{\KLx}[2]{(}{)}{%
#1\,\delimsize\|\,#2%
}

\newcommand{\KL}[2]{K(#1 \mid\mid #2)}

\newcommand\centroid[1]{\accentset{\circ}{#1}}

\newcommand*\mystrut[1]{\vrule width0pt height0pt depth#1\relax} 

\newcommand{\Ex}{\mathbb{E}_{{\bf X}}}
\newcommand{\Exy}{\mathbb{E}_{{\bf X}Y}}
\newcommand{\EXY}{\mathbb{E}_{\vectorX\vectorY}}

\newcommand{\EXYb}[1]{\mathbb{E}_{\mathbf{X}\mathbf{Y}}\left[ #1 \right]}

\newcommand{\EYsb}[1]{\mathbb{E}_{Y}\left[ #1 \right]}

\newcommand{\ED}{\mathbb{E}_{\scriptstyle D}}
\newcommand{\EDb}[1]{\mathbb{E}_{\scriptstyle D}\left[#1 \right]}

\DeclareMathOperator*{\argmin}{\arg\min}

\newcommand{\bregmanDomain}{\mathcal{Y}}

\newcommand*\diff{\mathop{}\!\mathrm{d}}

\newcommand\defeq{\mathrel{\overset{\makebox[0pt]{\mbox{\normalfont\tiny\sffamily def}}}{=}}}

\newcommand{\qbar}{\overline{q}}

\newcommand{\vectorq}{\mathbf{q}}
\newcommand{\vectorqbar}{\mathbf{\overline{\vectorq}}}

\newcommand{\vectorqtilde}{\mathbf{\widetilde{\vectorq}}}

\newcommand{\vectorp}{\mathbf{p}}
\newcommand{\vectorptilde}{\mathbf{\widetilde{\vectorp}}}

\newcommand{\vectory}{\mathbf{y}}
\newcommand{\vectorx}{\mathbf{x}}
\newcommand{\vectorz}{\mathbf{z}}
\newcommand{\vectorY}{\mathbf{Y}}

\newcommand{\vectorX}{\mathbf{X}}
\newcommand{\vectoreta}{\mathbf{\eta}}

\newcommand{\ystar}{y^*}
\newcommand{\Ystar}{Y^*}
\newcommand{\vectorystar}{\vectory^*}
\newcommand{\vectorYstar}{\vectorY^*}

\newcommand\averagei{ \frac{1}{m}\sum_{i=1}^m }

\newcommand{\Lzeroone}{\ell_{0/1}}

\newcommand{\generator}[1]{\ensuremath{\phi\left(#1\right)}}
\newcommand{\bregman}[2]{\ensuremath{\generator{#1} - \generator{#2} - \langle\nabla\generator{#2},  ( #1-#2 ) \rangle}}

\newcommand{\gradinv}[1]{\left[\nabla \phi\right]^{-1}\left(#1\right)}

\newcommand{\boldeta}{\ensuremath{\boldsymbol{\eta}}}

\newcommand{\gen}[0]{\ensuremath{\phi}}

\newcommand{\Bregman}[2]{\ensuremath{B_{\gen}\left(#1, #2\right)}}
\newcommand{\BregmanGen}[3]{\ensuremath{B_{#1}\left(#2, #3\right)}}

\newcommand{\relativeinterior}{\textnormal{\textrm{ri}}}

\newcommand{\bias}{\ensuremath{\overline{\textnormal{bias}}}}
\newcommand{\variance}{\ensuremath{\overline{\textnormal{variance}}}}

\newcommand{\sign}{\textnormal{sign}}

%% file: 0_abstract.tex
\begin{abstract}%
We present a theory of ensemble diversity, explaining the nature of diversity for a wide range of supervised learning scenarios.
This challenge 
has been referred to as the “holy grail” of ensemble learning, an open research issue for over 30 years.
Our framework reveals that diversity is in fact a {\em hidden dimension} in the bias-variance decomposition of the ensemble loss.
We prove a family of {\em exact} bias-variance-diversity decompositions, for a wide range of losses in both  regression and classification,
e.g., squared, cross-entropy, and Poisson losses. 
%
%
For losses where an additive bias-variance decomposition is not available (e.g., 0/1 loss) we present an alternative approach: quantifying the {\em effects} of diversity, which turn out to be dependent on the label distribution.
Overall, {\em we argue that diversity is
a measure of model fit, in precisely
the same sense as {bias} and {variance}}, but accounting for statistical dependencies {\em between} ensemble members. Thus, we should not be `maximising diversity' as so many works aim to do---instead, we have a bias/variance/diversity trade-off to manage.

\end{abstract}

%% file: 01_Introduction.tex
\section{Introduction}
\label{submission}

Ensemble methods have enabled state-of-the-art results for decades: from early industrial computer vision \citep{viola2001rapid} to the deep learning revolution \citep{krizhevsky2012imagenet}, and inter-disciplinary applications \citep{cao2020ensemble}.  
An accepted mantra is that ensembles work best when the individuals have a ``diversity'' of predictions---often induced by classical methods such as Bagging \citep{breiman1996bagging}, but diversity-encouraging heuristics are rife in the literature \citep{Brown2005Diversity}.
Given this, we trust that the ensemble will ``average out'' the errors of the individuals.
One reason for the popularity of such methods is clear: the very {\em idea} of ensembles is an appealing anthropomorphism, invoking analogies to human committees, and ``wisdom of the crowds''.  
However, such analogies have limitations. More formal approaches have been pursued, in particular for {\em quantifying} diversity.
It is obvious that we do not want all predictions to be identical;
and, it is equally obvious we do not want them to be different just for the sake of it, sacrificing overall performance.
We want something in-between these two---the so-called accuracy/diversity trade-off.
However, here we encounter the problem of formally defining ``diversity'' and its relation to ensemble performance.
{\em In general, there is no  agreement on how to quantify diversity}, except in the limited case of regression with an arithmetic mean ensemble \citep{Krogh1994Neural,ueda1996generalization}.
For classification and other scenarios, there are dozens of proposed diversity measures \citep{kuncheva2014}. A~comprehensive theory of ensemble diversity has been an open problem for over 30 years.\\ 

\paragraph{Motivation:}
Our primary motivation is to fill this `gap' in current ensemble theory, providing a solid foundation to understand and study ensemble diversity. However, there are also many practical reasons to pursue this.
Diverse ensembles can be more computationally efficient than single large models, with the {\em same} generalisation performance \citep{kondratyuk2020ensembling}.
Diverse ensembles are robust against adversarial attacks \citep{biggio2011bagging, pang2019improving}, and can counteract covariate shift \citep{sinha2020dibs}.
Advantages are also found in important application areas \citep{cao2020ensemble} and well beyond supervised learning \citep{carreira2016ensemble}.
It is important to note that these use-cases {\em do not follow a common approach}: they either adopt some measurement of diversity picked from historical literature, or propose their own novel metric.
There is, therefore, reason to pursue a unified theory, where diversity is derived from first principles. 

This challenge has proven non-trivial: surveys of progress can be found in \citet{dietterich2000ensemble,Brown2005Diversity,zhou2012ensemble}, and \cite{kuncheva2014}. 
Diversity is nowadays referred to as a heuristic with no precise definition,
and, it has been said:

\begin{quote}
{\em ``There is no doubt that understanding diversity is the holy grail in the field of ensemble learning''} \citep[Sec 5.1, page 100]{zhou2012ensemble}.\\
\end{quote}

\paragraph{Summary of our Results:}
In contrast to previous efforts which {\em define} novel diversity measures, we take loss functions and {\em decompose} them, exposing terms that {naturally} account for diversity.
We show that diversity is a {\em hidden dimension} in the {\em bias-variance decomposition} of the ensemble loss.  In particular, we prove {\em exact} bias-variance-{\em diversity} decompositions, applying for a broad range of losses, taking a common form:
\begin{equation}
     \textbf{expected~loss} ~=~ (\textbf{average bias)} ~+~ (\textbf{average variance}) ~-~ (\textbf{diversity}), \notag
\end{equation}
where diversity is a measure of ensemble member disagreement, independent of the label.
For the special case of squared loss, this is an alternative to a decomposition proposed by \citet{ueda1996generalization}, but expanding the formal notion of diversity to many other losses, e.g., the {\em cross-entropy}, and the Poisson loss.
For losses where an additive bias-variance decomposition does not exist (e.g., 0/1 loss) we present an approach to precisely quantify the {\em effects} of diversity, with the caveat that the effects are conditional on the label distribution.
%

Overall, {\em we argue that diversity is best understood as a measure of model fit, in precisely the same sense as {bias} and {variance}}, but accounting for statistical dependencies {\em between} ensemble members. Thus, we should not be `maximising' diversity as so many works claim to do---instead, we have a trade-off to manage.
With single models, we have the well-known {\em bias/variance} trade-off. With an ensemble we have a bias/variance/{\em diversity} trade-off, varying both with individual model capacity, {\em and} similarities {between} model predictions.

%% file: 01_Notation.tex
\section*{Notation}
For quick reference, we summarise the majority of notational conventions used in the paper.  More detailed definitions are given in the main body as appropriate.
Boldface font, e.g., $\bf y$, denotes a vector, whereas $y$ is a scalar.
Capital letters, e.g., $Y$, denote a random variable, where $y$ is a realisation of $Y$.  \\

\noindent For brevity, in the majority of the paper we leave it implicit that a model, $q$, is a function of an input ${\bf x}$, and is dependent on a training set sampled from a random variable $D$.

\begin{table}[!h]
\def\arraystretch{1.2}
\begin{tabular}{@{}ll@{}}
\toprule
\bf \underline{Symbol}        &   \bf \underline{Description} \\
 &    \\
$\vectorx\in \mathcal{X}\subseteq\mathbb{R}^d$   &   Input, in $d$ dimensions\\
$\vectory\in \mathcal{Y}\subseteq \mathbb{R}^k$   &   Label, in $k$ dimensions (or $y$ for a scalar)\\
$\{(\vectorx_j,\vectory_j)\}_{j=1}^n$ &   Training set of $n$ labelled examples\\

 &    \\
 
$P(\vectorx,y)$              &   Unknown data distribution over $(\vectorx,y)$\\
$\Exy [ \cdots ]$            &   Expectation w/r $P(\vectorx,y)$, i.e., over all possible test examples $(\vectorx,y)$     \\
                             &  $\Exy \big[ \cdots \big] ~\defeq~ \int P_{\vectorX}({\bf x} )  \int  
P_{Y|\vectorX}(y ~|~ {\bf x}) \big[ \cdots \big]  \diff {y} \diff {\bf x}$\\
 &    \\

$D$                          &   Random variable $D\sim P(\vectorx,y)^n$, over i.i.d. training sets of size $n$ \\

$\ED [ \cdots ]$            &   Expectation over all possible training sets $\{(\vectorx_j,\vectory_j)\}_{j=1}^n$ drawn from $D$ \\

&    \\ 

$q(\vectorx)$                &  Predictive model $q$, mapping $\vectorx \rightarrow y$.  If output is a vector then ${\bf q}({\bf x})$ \\
$\ell(y,q({\bf x}))$          &  Loss function for label $y$ and prediction $q({\bf x})$ \\
$\ell_{0/1}(y,q({\bf x}))$              & The 0/1 loss function: returns 1 if $y\neq q({\bf x})$, and 0 otherwise.\\
$R(q)$ & Risk of model $q$, defined $R(q):= \Exy[\ell(y,q({\bf x}))]$\\

&    \\ 

$m$   &       Number of ensemble members\\
$q_i$ &       The $i$th member of an ensemble, $\{{q}_i\}_{i=1}^m$\\
$\qbar$ &      Centroid combiner rule, defined $\qbar :=\argmin_{z\in\mathcal{Y}} \averagei[\ell(z,q_i)]$\\
$\centroid{q}$ &       Centroid of a model distribution w/r $D$, defined $\centroid{q}:=\argmin_{z\in\mathcal{Y}} \ED[\ell(z,q)]$\\

&    \\ 

$\phi(\vectory)$    &    Strictly convex generator function for a Bregman divergence\\
$\nabla\phi(\vectory)$    &    Gradient vector of generator $\phi$ w/r argument $\vectory$\\
$B_\phi({\bf y},{\bf q})$   &    Bregman divergence between points ${\bf y},{\bf q} \in \mathbb{R}^k$, using generator $\phi$\\

 &    \\ \bottomrule
\end{tabular}
\caption{Notation used in the paper.\label{tab:notation}}
\end{table}

\newpage

%% file: 02_ProblemStatement.tex
\section{Problem Statement: What is Ensemble Diversity?} \label{sec:problemstatement}

\cite{hansensalamon1990} proposed a methodology to train multiple neural networks, encouraging different models by providing each with a different training data subset. 
Many subsequent papers followed this ``parallel" strategy: an early work being \citet{perrone1992networks}, but perhaps more well-known are {\em Bagging} \citep{breiman1996bagging}, and {\em Random Forests} \citep{breiman2001random}.
{\em Boosting} algorithms  \citep{schapire1998boosting} exploit a similar principle, but construct models \emph{sequentially}.
These approaches, parallel and sequential (see Figure~\ref{fig:parallellvssequential}), are the most common schemes to learn ensembles \citep{kuncheva2014}, and there are many variations, e.g., selecting members from a pre-constructed candidate pool, each built by independent (parallel) teams working on the same task but with different data.
\begin{figure}[ht]
    \centering
    \includegraphics[width=0.9\textwidth]{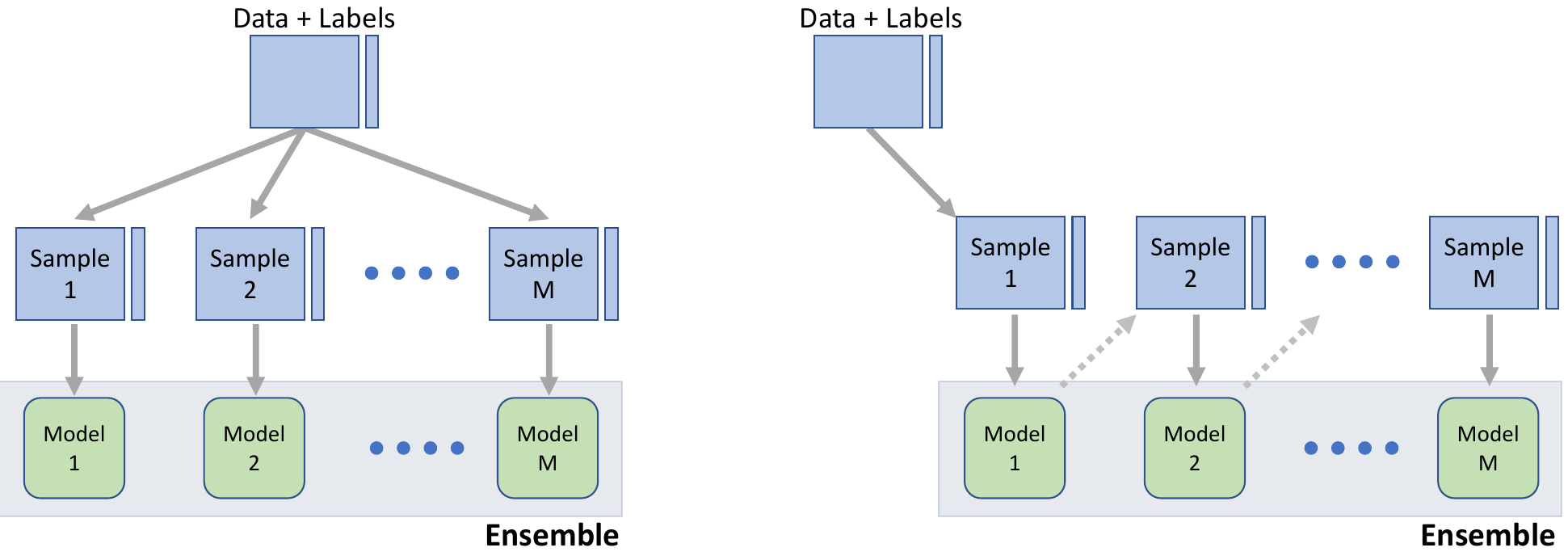}
    \caption{Parallel vs sequential ensemble construction. Both can be seen as creating ``diverse" models in some sense---either {\em implicitly} (independently re-sampling the training data), or {\em explicitly} (re-sampling according to the errors of earlier models).} 
    \label{fig:parallellvssequential}
\end{figure}

\noindent So why do these strategies work? Both can be understood heuristically in terms of ``diversity", in the sense coined by \cite{opitz1996generating}, referring to differences in generalisation behavior among a group of models. In a review, \citet{dietterich2000ensemble} explains:
\begin{quote}
{\em ``An accurate classifier is one that has an
error rate of better than random guessing on new x values. Two classifiers are diverse if they make different errors on new data points.''} \citep{dietterich2000ensemble}
\end{quote}
In this sense, both approaches foster diversity---either {\em implicitly} by randomly perturbing the data for each model, or {\em explicitly} by constructing each data set to address the errors of previous models
\citep{Brown2005Diversity}.
The implicit approach has been widely adopted in deep learning: \cite{Goodfellow-et-al-2016} note that sources of randomness in the initialisation of deep networks are often enough to cause them to make partially independent errors.
%
Given the success of ensembles, there have been many attempts to explain {\em why} they work, in terms of the diversity.
\citet{Goodfellow-et-al-2016} writes,
\begin{quote}
{\em ``The reason that model averaging works is that different models will usually not make all the same errors on the test set''}. \citep[p. 249]{Goodfellow-et-al-2016}
\end{quote}
Whilst this is true, we desire a more formal treatment.

\paragraph{What are we looking for?}
A theory of diversity would ideally have three key ingredients:
\begin{enumerate}
   \item a definition of diversity as a measure of disagreement between the ensemble members;
   \item this measure should have a clear relation to the overall ensemble error; and,
   \item the theory should have a clear relation to previously established results,
   and expand our understanding to a wider range of learning scenarios.
\end{enumerate}
The first point concerns the definition of diversity itself.   An interesting question is whether the measure is a function solely of the ensemble member outputs, or whether it can involve the label distribution as well. 
\citet{kuncheva2014} refers to measures that rely on the true label as `oracle' diversity measures.
{\em Both outcomes would be an interesting scientific conclusion}: the label-independent case showing that diversity is solely a function of the ensemble itself, and the label-dependent case showing that it is a class-conditional phenomenon.

The second point ensures we can interpret what effect diversity has on our ultimate objective: reducing the ensemble error. 
The third point ensures that the new theory contributes to knowledge, in the sense that it expands the depth/breadth of our understanding.  This also relates to the only known scenario where the challenge of defining diversity might be considered a ``solved'' problem: regression using squared loss, with an arithmetic mean ensemble.  We now review this.
%

\paragraph{Known results for regression ensembles:}
\citet{Krogh1994Neural} showed that, for an arithmetic mean combiner, using squared loss, the ensemble loss {\em is guaranteed to be less than or equal to the average individual loss}.
\begin{theorem}[Ambiguity decomposition, Krogh \& Vedelsby, 1994]
Given a label $y\in\mathbb{R}$, a member prediction  $q_i({\bf x})$, and an ensemble $\bar{q}({\bf x})=\averagei q_i({\bf x})$, we have,
\begin{equation}
    \big(\bar{q}({\bf x}) - y\big)^2 ~=~ \averagei \big(q_i({\bf x}) - y\big)^2 - \averagei \big(q_i({\bf x}) - \bar{q}({\bf x})\big)^2.
    \label{eq:ambiguity}
\end{equation}
\end{theorem}
The left hand side is the ensemble loss at a single test point $({\bf x},y)$. The first term on the right is the average individual loss.  The second is known as the {\em ambiguity}---measuring the disagreement of individuals, as a spread around the ensemble prediction. Since this is non-negative, it guarantees the ensemble loss will be less than or equal to the average loss.

This result is often {\em erroneously} cited as the reason why all ensembles work.
However, the form above applies {\em {only if} we use the squared loss with an arithmetic mean combiner.}  If we use squared loss with a different combiner, the result no longer holds.
A deeper understanding came from \citet{ueda1996generalization}---though under the same loss/combiner assumptions. They extended the classical bias-variance theory of \citet{geman1992neural} to show that the {\em expected} squared loss of the ensemble decomposes into three terms, involving the bias, variance, and {\em covariance} of the models. Here, the covariance captures the notion of `diversity'. More details on this will be provided later.

As mentioned, the expression above does not apply beyond squared loss with the arithmetic mean combiner. %
A significant community effort has been directed to find corresponding notions of diversity for classification problems.  We review this next.

\paragraph{Known results for classifier ensembles:}
For classification problems, we might have estimates of the class probability distribution, or just labels---usually combined by averaging probabilities, or majority voting.
%
\citet{tumer1996error} showed that the correlation between pairs of averaged class probabilities had a simple relationship to the overall ensemble classification error, at least in a region close to decision boundaries.
%
%
\citet{brown2009information} and \citet{zhou2010multi} proposed information theoretic analyses, showing that diversity manifests as both low- and high-order interactions between ensemble members.
\citet{buschjager2020generalized} used a Taylor approximation on twice-differentiable losses, showing an exact decomposition only when higher derivatives are zero, e.g.,squared loss, but not cross-entropy.
Similarly, \citet{ortega2022diversity} decomposed upper bounds on losses, again only obtaining an equality for squared loss. 
\citet{kuncheva2003measures} considered diversity measures for their correlation to the ensemble error. They investigated numerous discrepancy metrics in the form $\delta(q_i({\bf x}),q_j({\bf x}))\in \mathbb{R}$, which are averaged over all pairs of ensemble members at each test point ${\bf x}$:
\begin{equation}
    \textnormal{diversity}(q_1,..,q_m, {\bf x}) = \frac{1}{m(m-1)}\sum_{i=1}^m\sum_{j\neq i} \delta(q_i({\bf x}),q_j({\bf x})).
\end{equation}
The ensemble diversity is this quantity averaged over a validation dataset.  The ensemble diversity is evaluated on its empirical correlation to the  ensemble performance, and seen as more successful if it has high correlation, illustrated in Figure~\ref{fig:diversity_toy_plot}.
\begin{figure}[h]
    \centering
    \includegraphics[width=13.2cm]{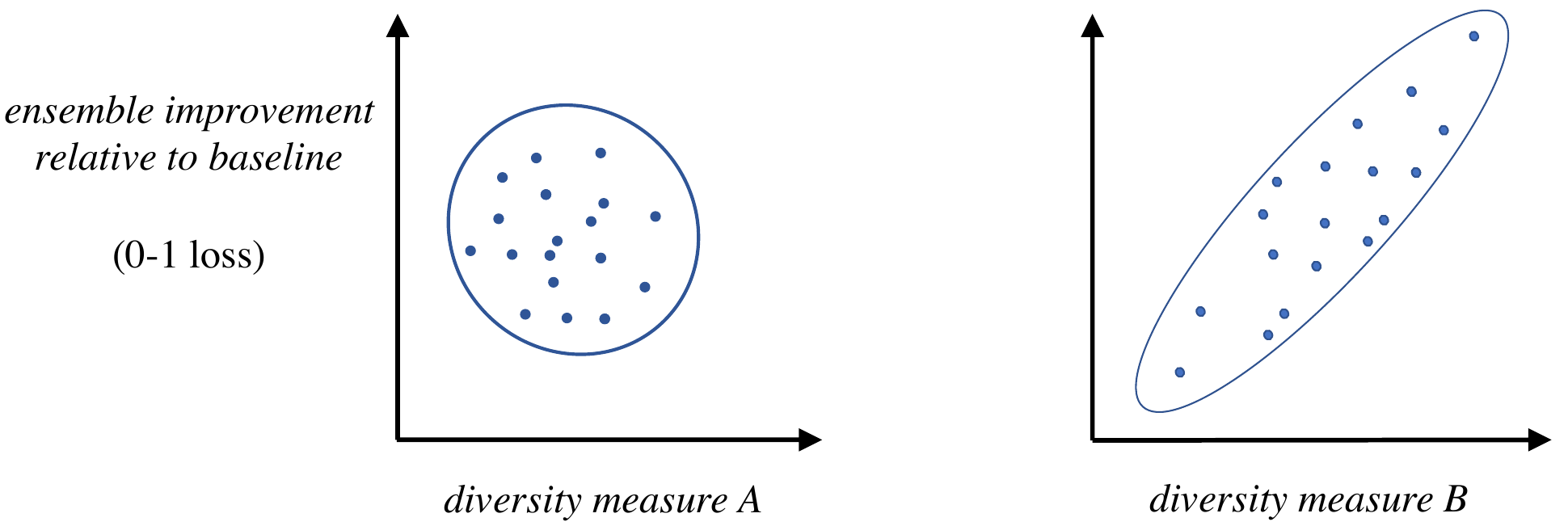}
    \caption{Accuracy/diversity for two (hypothetical) diversity measures. Measure~B (right) is more desirable, as it has stronger correlation to performance improvement.} 
    \label{fig:diversity_toy_plot}
\end{figure}

Several measures (including also non-pairwise measures) were explored,
with no single measure proving more successful than any other.
%
Almost 20 years on, novel diversity heuristics/measures are still being proposed, e.g., \citet{jan2019novel,wu2021boosting}. 

\paragraph{Our approach to the problem:}
As above, diversity is often discussed in a dichotomy of classification vs. regression. 
One of our primary observations is that {\em this high-level dichotomy is insufficient to fully describe the nature of diversity}. Instead, we must consider the {\em loss} function at hand. For regression, this might be squared loss, but also Poisson, or others. For classification, this may be 0/1, or cross-entropy.
%
We build on the strong foundation of bias-variance theory \citep{geman1992neural,james1997generalizations,heskes1998bias,pfau2013},
and {\em decompose} the losses, naturally exposing notions of diversity.

%% file: 03_VeryShortIntro.tex
\newpage

\section{A Very Short Introduction to Bias-Variance Decompositions}
\label{sec:veryshortBVintro}

We review literature on {\em bias-variance} decompositions: the statistical setting is a standard supervised learning scenario.
We assume a training set $\{(\vectorx_j, y_j)\}_{j=1}^{n}$ drawn from a random variable $D\sim P(\vectorX, Y)^n$.
From this, we learn a model,~$q$.
For brevity (see \autoref{tab:notation})
in the remainder of the paper 
we leave it implicit that $q$ is a function of $\vectorx$, and is always dependent on data sampled from~$D$.
%
%
\noindent We denote the {\em risk} as $R(q) \defeq \Exy \left[ \ell\left(Y,q\right) \right]$,
%
where the label $Y$ is a random variable, thus potentially noisy.
If $\ell(y,q)=(y-q)^2$, this is the {\em squared} risk, and the
Bayes-optimal prediction is the conditional mean
$\Ystar \defeq \mathbb{E}_{Y|\vectorX}[Y]$.
\citet{geman1992neural} considered the {\em expected squared risk}, over possible training sets drawn from $D$, showing the following three term decomposition.
%
\begin{equation}
    \underbrace{\mystrut{12pt} \ED\Big[ \Exy  [ (Y-q)^2 ]\Big] }_{\substack{\textnormal{expected~squared}~\textnormal{risk}}} 
        = 
        \Ex\Bigg[ \underbrace{\mystrut{12pt} \mathbb{E}_{Y\mid \vectorX}\left[ (Y-\Ystar)^2 \right] }_{\textnormal{noise}} 
        +  \underbrace{\mystrut{12pt}\Big(\ED \left[ q\right]-  \Ystar\Big)^2}_{\textnormal{bias}} 
        + \underbrace{\mystrut{12pt} \ED \left[ \Big( q - \ED \left[ q \right] \Big)^2 \right] }_{\textnormal{variance}} \Bigg].  \label{bvsqloss}
\end{equation} 
%
%
The first term
is the irreducible {\em noise} in the problem, independent of any model parameters.
The second is the {\em bias}\footnote{{\noindent\em Is it `squared bias', or just `bias'?} 
Geman et al's bias term is often referred to as ``squared bias''.  However, the square is in fact an artefact of using the squared  loss, and is not present in the decompositions of other losses. Thus, the term ``squared bias'' {\em is a misnomer}, and throughout the paper we use simply ``bias''.   Further exposition on this point is provided in Appendix \ref{app:squaredbias}.}, defined as the loss of the expected model against the Bayes-optimal prediction. 
%
%
The third is the {\em variance}, capturing variation in $q$ due to different training sets.
Note that we can use the term {\em bias} (correspondingly {\em variance}) to indicate the value at a point $({\bf x},y)$, or averaged over a distribution---the intention will be clear from context.
\noindent The ideas are often explained with a dartboard diagram, as in~\autoref{fig:BVdartboard}. 
\begin{figure}[ht]
    \centering
    \includegraphics[width=.219\textwidth]{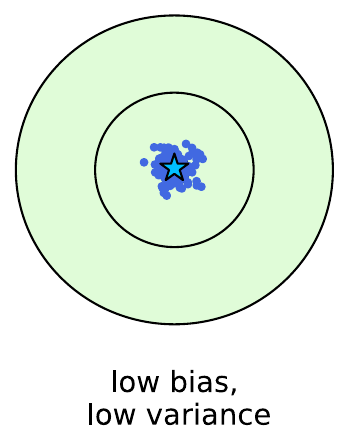}
    \includegraphics[width=.219\textwidth]{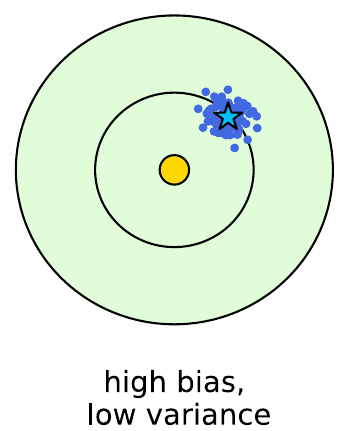}
    \includegraphics[width=.219\textwidth]{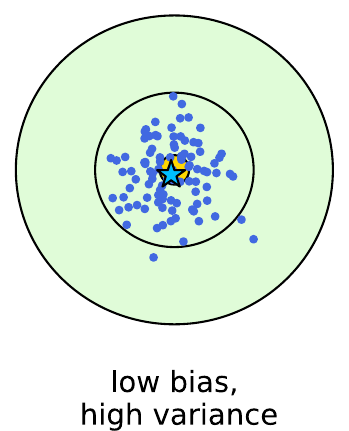}
    \includegraphics[width=.219\textwidth]{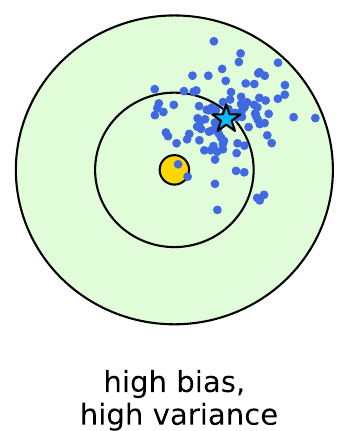}
    \caption{The classic dartboard analogy for explaining bias and variance. 
    \label{fig:BVdartboard}}
\end{figure}

The bullseye (yellow circle) is the target for a single {\em test} point, and each blue dot is a prediction from a model 
given a different {\em training} set.
A model with high bias/low variance
will be on average far from the target, and insensitive to small data changes.
A model with low bias/high variance
will have an expected value  
close to the target, but be very sensitive to data changes, meaning any given model is likely to overfit. 
The random variable $D$ does not {\em have} to be over data, but can be any stochastic quantity involved in learning the model, e.g., random weights in a neural network.\\


\paragraph{Bias-variance decompositions apply for more than just squared loss.}  
\citet{geman1992neural} is a widely appreciated result.
It is less well known that similar decompositions hold for other losses, e.g., the KL-divergence of class probability estimates \citep{heskes1998bias}.
Assume $\vectory$ is a one-hot vector of length $k$, and $\vectorq\in\mathbb{R}^k$ is our model's predicted distribution over the $k$ labels. 
The expected risk decomposes\footnote{Heskes (1998) presents a slight variation, though the expressions are equivalent. See Appendix \ref{app:decomposiingCrossEntropy}.} in a similar three-part form:
\begin{equation}
    \underbrace{\mystrut{12pt} \ED\Big[ \EXY\left[ K\left(\vectorY\mid\mid\vectorq\right) \right] \Big] }_{\substack{\textnormal{expected}~\textnormal{risk}}} 
        = 
        \Ex\Bigg[ \underbrace{\mystrut{12pt} \mathbb{E}_{\vectorY\mid \vectorX}\left[ K\Big(\vectorY\mid\mid\vectorYstar\Big) \right] }_{\textnormal{noise}} 
        +  \underbrace{\mystrut{12pt}K\Big( \vectorYstar\mid\mid \centroid{\vectorq} \Big)}_{\textnormal{bias}} 
        + \underbrace{\mystrut{12pt} \ED \left[ K\Big( \centroid{\vectorq} \mid\mid \vectorq\Big) \right] }_{\textnormal{variance}} \Bigg],  \label{bv:KL}
\end{equation}
where $\centroid{\vectorq}
\defeq \argmin_{\vectorz\in\mathcal{Y}} \mathbb{E}_D\Big[K(\vectorz\mid\mid\vectorq)\Big]  = Z^{-1}\exp(\ED[\ln \vectorq])$
is a {\em normalized geometric mean} at each input vector $\vectorx$.
%
%
This can also be illustrated as a ``dartboard" analogy, in Figure~\ref{fig:simplexes}.

\begin{figure}[!h]
    \centering
    \includegraphics[width=.22\textwidth]{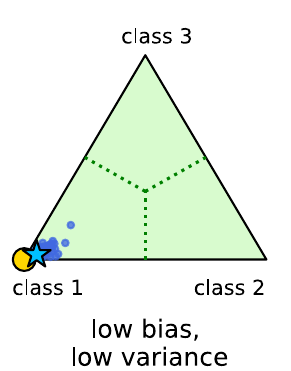}
    \includegraphics[width=.22\textwidth]{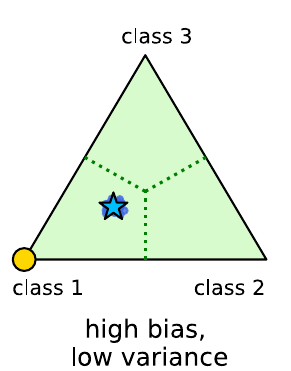}
    \includegraphics[width=.22\textwidth]{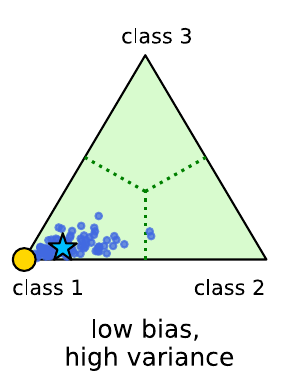}
    \includegraphics[width=.22\textwidth]{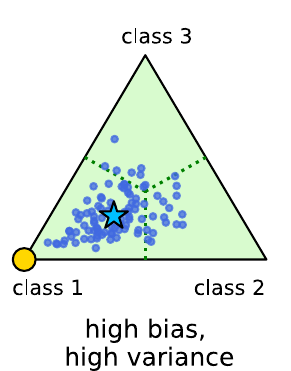}
    \caption{Dartboard diagram illustrating bias/variance for the KL-divergence.}
    \label{fig:simplexes}
\end{figure}


%
\paragraph{Bias-variance decompositions do not apply for all losses.}

A commonality of the decompositions above is that the variance  is {\em independent} of the label.  
A further commonality is that the bias is the loss of a predictor $\centroid{q}
\defeq \argmin_{z\in\mathcal{Y}} \mathbb{E}_D\Big[\ell(z,q)\Big]$, not dependent on any particular training set.
%
%
%
We refer to this form as an {\em additive} bias-variance decomposition.
It turns out that this does not hold for all losses, e.g., 0/1 loss---in this case $\centroid{q}$
is the {\em mode} of the random variable,
but, we have an {\em inequality}:
\begin{equation}
    \ED\Big[ \Exy[\Lzeroone(Y,q)] \Big] ~\neq~ \Ex\Big[ \mathbb{E}_{Y|\vectorX}[\Lzeroone(Y,Y^*)] + \Lzeroone(Y^*,\centroid{q}) + \ED[ \Lzeroone(\centroid{q},q) ] \Big].
\end{equation}
Many authors have presented alternative 0/1 decompositions \citep{kohavi1996bias,tibshirani1996bias,james1997generalizations,heskes1998bias,domingos2000unified}.
An excellent review can be found in \citet{geurts2002contributions}.
A common theme is to either sacrifice the simple additive form, or have a formulation of variance that is dependent on the label distribution. 
The necessary and sufficient conditions for an additive bias-variance decomposition of the form mentioned above, are an open research question. \\

Together, the theories outlined above have been used to understand the nature of model fitting---the so-called `bias-variance' trade-off. 
\noindent In the following section we build upon all these theories to understand the phenomenon of ensemble diversity.

%% file: 04_Unified.tex

\newpage

\section{A Unified Theory of Ensemble Diversity}
\label{sec:unified_diversity}

Ensemble ``diversity'' is a popular, but variously defined idea.  Bias and variance are clear-cut, and precisely defined. 
It makes sense to build stronger bridges between the two.  Our approach is exactly this, revealing diversity as a {\em hidden dimension} in the bias-variance decomposition of an ensemble. 
We argue that diversity can be considered in exactly the same way as bias/variance---simply another quantifiable aspect of model fitting.

\subsection{Diversity as a Hidden Dimension of the Bias-Variance Decomposition}\label{subsec:hidden}
Our approach yields a `unified theory' in the sense that the {\em same} methodology applies for many different losses, in classification/regression scenarios. 
We first define a {\em generalised} (additive) bias-variance decomposition, in that it generalises Equations~\eqref{bvsqloss} and~\eqref{bv:KL}.

\begin{definition}[Generalised Bias-Variance Decomposition]\label{def:gen_bv}
For a loss $\ell$, we define $\centroid{\vectorq} \defeq \argmin_{\vectorz\in\mathcal{Y}} \mathbb{E}_D \Big[\ell(\vectorz,\vectorq)\Big]$ as the `centroid' of the model distribution.
If the following form holds, we refer to it as a generalised bias-variance decomposition.
\begin{align}
       \underbrace{\mystrut{.5em}\mathbb{E}_D\Big[\EXY[\ell(\vectorY, \vectorq)]\Big]}_{\textnormal{expected risk}} = \Ex\Big[ \underbrace{\mystrut{.5em}\mathbb{E}_{\vectorY|\vectorX}[\ell(\vectorY,\vectorYstar)]}_{\textnormal{noise}} + \underbrace{\mystrut{.5em}{\ell(\vectorYstar, \centroid{\vectorq})}}_{\textnormal{bias}} + \underbrace{\mystrut{.5em}\EDb{\ell (\centroid{\vectorq},\vectorq)}}_{\textnormal{variance}} \Big] \label{eq:unified_bv}.
\end{align}
\end{definition}
%
When $\ell(y,q)=(y-q)^2$, we have $\centroid{q}=\ED[q]$ and this reduces to Equation~\eqref{bvsqloss}. Alternatively, with $\ell(\vectory,\vectorq)=K(\vectory\mid\mid\vectorq)$, we have $\centroid{\vectorq} = Z^{-1}\exp(\ED[\ln \vectorq])$, and it reduces to Equation~\eqref{bv:KL}.    
Importantly, this same additive form also holds for the even broader family of {\em Bregman divergences} \citep{bregman1967relaxation}.  
This family, reviewed in the next section,
provides us a convenient analytical form, which we will use to derive several interesting properties.

We now take our first step toward ensembles. We highlight that the {\em ambiguity} decomposition, Equation~\eqref{eq:ambiguity}, can be proven by a very similar method to the bias-variance decomposition.  In fact: if one exists, then the other must also, outlined in Appendix~\ref{app:specialcase}.
For a loss $\ell$, if Equation~\eqref{eq:unified_bv} holds, we can state the corresponding ambiguity decomposition.

\begin{proposition}[Generalised Ambiguity Decomposition]\label{prop:gen_ambig}
Assuming loss $\ell$ admits a bias-variance decomposition then, for an ensemble $\{\vectorq_i\}_{i=1}^m$, the ambiguity decomposition is,
\begin{align}
    \underbrace{\mystrut{1.3em}\ell(\vectory,\overline{\vectorq})}_{\textnormal{ensemble loss}} = \underbrace{\mystrut{1.3em}\averagei \ell(\vectory,\vectorq_i)}_{\textnormal{average loss}} - \underbrace{\mystrut{1.3em}\averagei  \ell(\overline{\vectorq},\vectorq_i)}_{\textnormal{ambiguity}},
    \label{eq:unified_ambiguity}
\end{align}
where $\vectorqbar \defeq \argmin_{\vectorz\in\mathcal{Y}} \averagei \ell(\vectorz,\vectorq_i)$ is the ensemble combination.

\end{proposition}

We highlight that the ensemble combiner is defined {\bf as a property of the loss.} For squared loss, this results in $\qbar = \averagei q_i$, the commonly used arithmetic mean combiner. For KL-divergence, this results in $\vectorqbar = Z^{-1}\prod_{i=1}^m\vectorq_i^{1/m}$, i.e., a normalised geometric mean, also known as the `product rule' \citep{kuncheva2014}.
%
%
\citet{audhkhasi2013ambiguity} and \citet{jiang2017generalized} proposed  ambiguity decompositions for classification losses, though both  assumed the combiner must be an arithmetic mean.
In contrast, here the combiner is defined {\em in terms of the loss} $\ell$. We refer to this as the {\em centroid combiner rule},  defined  formally below.

\begin{definition}[Centroid Combiner rule]\label{def:centroidcombiner}
For a point $(\vectorx,\vectory)$, given a set of predictions $\{\vectorq_i\}_{i=1}^m$, the centroid combiner $\vectorqbar$ is the minimizer of the averaged loss $\ell({\bf z},{\bf q}_i)$, over all ensemble members.
\begin{equation}
\vectorqbar ~\defeq~ \argmin_{\vectorz\in\mathcal{Y}} \averagei \ell({{\bf z}},{{\bf q}_i}).
\label{eq:centroidcombiner}
\end{equation}
\end{definition}
For each given loss function, the centroid combiner {\em could} be derived separately.  However, in the next section we will meet the family of {\em Bregman divergences} \citep{bregman1967relaxation}, which provide us a common analytic form for the centroid combiner.
%
We can now present our main result,  Theorem~\ref{the:gen_bvd}, a generalised bias-variance-{\em diversity} decomposition.

\begin{restatable}[Generalized Bias-Variance-Diversity decomposition]{theorem}{thegenbvd}
\label{the:gen_bvd}
Consider a set of models $\{\vectorq_i\}_{i=1}^m$, evaluated by a loss $\ell$. Assuming a bias-variance decomposition holds in the form of Definition \ref{def:gen_bv}, the following decomposition also holds.
%
\begin{eqnarray}  \hspace{-0.5cm}\ED\Big[\EXY[\ell(\vectorY, \overline{\vectorq})]\Big] &=& \notag\\
    &&\hspace{-4cm} \Ex\Bigg[ \underbrace{\mystrut{1.5em}\mathbb{E}_{\vectorY|\vectorX}[\ell(\vectorY,\vectorYstar)]}_{\textnormal{noise}}
    +
    \underbrace{\mystrut{1.5em}\averagei \ell(\vectorYstar,\centroid{\vectorq}_i)}_{\textnormal{average bias}}
    + \underbrace{\mystrut{1.5em}\averagei \EDb{\ell(\centroid{\vectorq}_i,\vectorq_i)}}_{\textnormal{average variance}}
    - \underbrace{\mystrut{1.5em}\EDb{\averagei  \ell(\overline{\vectorq},\vectorq_i)}}_{\textnormal{diversity}}\Bigg],
\end{eqnarray}
where $\centroid{\vectorq} \defeq \argmin_{\vectorz\in\mathcal{Y}} \mathbb{E}_D \Big[\ell(\vectorz,\vectorq)\Big]$ and the combiner is $\overline{\vectorq} \defeq \argmin_{\vectorz\in\mathcal{Y}} \averagei \ell(\vectorz,\vectorq_i)$.
\end{restatable}

This has decomposed the expected ensemble risk into: the noise, the average bias, the average  variance, and the expected ambiguity.  It is this expected ambiguity term that we consider as the ensemble {\em diversity}, which we highlight has the opposite sign to other terms. We can make two important observations.\\

\noindent {\bf The ensemble combiner $\vectorqbar$ is a centroid.}  In a practical scenario, the loss of the ensemble is chosen for whatever the task is: 0/1 loss,
cross-entropy, or MSE, or many others. The combiner rule can also be chosen as almost anything: voting, arithmetic
mean, weighted combinations, geometric mean, etc.
The result above shows that, if we impose a constraint between these choices (i.e., Definition~\ref{def:centroidcombiner}) we enable a decomposition which naturally exposes a diversity term. {\em We emphasize that using this makes no claim on the {\em empirical} behavior of the ensemble, but simply on their theoretical properties}.\\

\noindent {\bf There is now a bias/variance/diversity trade-off.}  As individual models increase in capacity, their average bias decreases.  Without regularisation, their average variance would increase.  However, these terms determine only part of the ensemble behavior. The final part is the diversity. 
A critical point here is {\em diversity always subtracts from the expected risk}.  This is not to say that greater diversity always reduces expected risk---it only reduces it given a fixed bias and variance. Ultimately, the three-way trade-off of {\em bias/variance/diversity} is what determines the overall ensemble performance.
It is worth highlighting that diversity is defined similarly to bias/variance, involving an expectation over $D$, as opposed to being a property of a single training run.
We also highlight the result applies for both dependent and independent training (where $D$ is a vector of $m$ random variables, see Appendix~\ref{app:dependent} for details) and also for schemes where $D$ defines a pool of pre-constructed models.

\newpage
\subsection{What if a bias-variance decomposition doesn't hold?}

Definition 2 does not apply for all losses, e.g., for 0/1 loss, or absolute loss.  
This creates an obvious challenge for our framework.
We relied on the existence of such a decomposition, deriving
diversity as a `hidden dimension'.
However, there is still a way forward.\\

\noindent {\bf Quantifying the Effect of Bias/Variance/Diversity}. \citet{james1997generalizations}
present a decomposition which applies for {\em any} loss, and retains the simple additive form.   They achieve this by distinguishing the {\em measurement} of variance (which is independent of the label) from its {\em effect} on the expected risk (which is dependent on the label),
with the analogous pair of quantities for the bias. The measurement and the effect are not necessarily the same quantity.  
They considered the {\em measurement} to be the `natural' form in Definition~\ref{def:gen_bv}.
They then defined the {\em bias-effect}\footnote{J\&H called this the {\em systematic} effect. For consistency within our work, we prefer the term {\em bias}-effect.}, and the {\em variance-effect}.
For compact notation, we average these over $P(\vectorx, \vectory)$. 
Using $R({\bf q})\defeq \Exy[\ell({\bf y},{\bf q})]$ as the risk of ${\bf q}$, their terms are:
%
\begin{eqnarray}
    \textnormal{bias-effect} &&\defeq R(\centroid{\bf q})-R({\bf y}^*),\label{eq:biaseffect_alt}\\
    \textnormal{variance-effect} &&\defeq \ED\Big[ R({\bf q})-R(\centroid{\bf q}) \Big],
\end{eqnarray}   
where $\centroid{\vectorq} \defeq \argmin_{\vectorz\in\mathcal{Y}} \mathbb{E}_D \Big[\ell(\vectorz,\vectorq)\Big]$, and ${\bf y}^*\defeq\argmin_{\vectorz\in\mathcal{Y}} \mathbb{E}_{\vectorY|\vectorX=\vectorx} \Big[\ell(\vectory,\vectorz)\Big]$ is the Bayes-optimal predictor\footnote{We acknowledge a slight abuse of notation, denoting the Bayes predictor $\vectory^*$ either as a function at each point $\vectorx$, or a vector in $\mathbb{R}^k$, as needed. The intention will be made clear from context.} at each point $\bf x$.
These quantify the {\em effect on the risk} for using one predictor instead of another.
The {\em bias-effect} is the change in risk, for the centroid $\centroid{\vectorq}$ versus the Bayes-optimal predictor ${\bf y}^*$. 
The {\em variance-effect} is the change in risk for a model $\bf q$ versus the centroid, averaged over the distribution of $D$.
For losses where Definition \ref{def:gen_bv} holds, the
measurement and the effect are the same, i.e., variance-effect is equal to variance, and bias-effect is equal to bias.
For example, with squared loss and $y\in\mathbb{R}$,
Equation~\eqref{eq:biaseffect_alt}
is equal to $\Exy[(\mathbb{E}_D[q]-y^*)^2]$. 
For general losses (e.g., 0/1 loss, absolute loss), this is not the case, and the {\em effect} terms are different numerical quantities.
With these definitions, James \& Hastie note the following decomposition.

\begin{theorem}[Bias-Variance Effect decomposition, James \& Hastie 1997]
For a loss $\ell:\mathcal{Y}\times\mathcal{Y}\rightarrow \mathbb{R}$, denoting the centroid as $\centroid{\bf q} \defeq \argmin_{\vectorz\in\mathcal{Y}} \EDb{\ell({\bf z},{\bf q})}$, the expected risk decomposes as follows:

\begin{equation}
      \ED\Big[R({\bf q})\Big] =   \underbrace{\mystrut{1em}R({\bf y}^*)}_{\textnormal{noise}} + \underbrace{\mystrut{1em}R(\centroid{\bf q})-R({\bf y}^*)}_{\textnormal{bias-effect}} + \underbrace{\mystrut{1em}\ED[ R({\bf q})-R(\centroid{\bf q})]}_{\textnormal{variance-effect}}
\end{equation}
\end{theorem}

This decomposition holds for  any
loss, since terms on the right always cancel to give the left side.
A notable point here is that the variance-effect is dependent on the label $y$, whilst the variance itself is a label-independent quantity.  Furthermore, whilst bias-effect is non-negative, the variance-effect can be a {\em negative} quantity.
From this, using the same relation exploited for Proposition~\ref{prop:gen_ambig}, we can state a corresponding ambiguity (effect) decomposition.

\begin{restatable}[Ambiguity-Effect Decomposition]{proposition}{ambiguityEffect}\label{prop:ambig_effect}
Given a loss $\ell:\mathcal{Y}\times\mathcal{Y}\rightarrow\mathbb{R}$, and an ensemble $\{\vectorq_i\}_{i=1}^m$ using the centroid combiner $\vectorqbar \defeq \argmin_{\vectorz\in \mathcal{Y}}\Big[\averagei\ell({\bf z},\vectorq_i)\Big]$, we have a decomposition:
\begin{equation}
\underbrace{\mystrut{1.5em}R(\vectorqbar)}_{\textnormal{ensemble~risk}} = \underbrace{\mystrut{1.5em}\averagei R(\vectorq_i)}_{\textnormal{average risk}} - \underbrace{\mystrut{1.5em}\averagei \Bigg[ R(\vectorq_i) - R(\vectorqbar) \Bigg]}_{\textnormal{ambiguity-effect}}. \label{eq:ambig_effect_alt}
\end{equation}
\end{restatable}

We re-emphasize that this applies for any loss. This is again quite obvious, as terms on the right cancel to give the left hand side.
For losses where Definition \ref{def:gen_bv} holds, this generalises ambiguity decompositions stated earlier.
For example, when $\ell$ is the squared loss,
it reduces to \citet{Krogh1994Neural}.
More generally, it reduces to \autoref{prop:gen_ambig}, the generalised ambiguity decomposition.

The powerful nature of this expression is that we can also consider losses where Definition~\ref{def:gen_bv} does {\em not} hold, e.g., 0/1 or absolute loss.
%
For the 0/1 loss, the centroid is the {\em mode} of the random variable \citep{domingos2000unified}. For a finite set of ensemble members, this means $\qbar$ is a majority vote of the individuals in the ensemble.  Or, with the absolute loss, $\qbar$ is the {\em median} prediction of the ensemble members.

Given all of the above, we can now present a novel decomposition
for the expected risk of an ensemble using the centroid combiner, isolating the {\em effect} of diversity. 

\begin{restatable}[Bias-Variance-Diversity effect decomposition]{theorem}{diversityEffect}\label{the:diversity_effect}
Given an ensemble of models $\{\vectorq_i\}_{i=1}^m$ combined by the centroid combiner $\vectorqbar \defeq \argmin_{\vectorz\in\mathcal{Y}}\Big[\averagei\ell({\bf z},\vectorq_i)\Big]$, using any loss $\ell$, the expected risk of the ensemble decomposes,
\begin{align*}
    ~\ED\Big[R(\vectorqbar)\Big] ~=~\\& &\hspace{-2.95cm}\underbrace{\mystrut{1.25em}R({\bf y}^*)}_{\textnormal{noise}} + \underbrace{\averagei \Big[ R(\centroid{\bf q}_i) - R({\bf y}^*) \Big] }_{\textnormal{average bias-effect}} 
     + \underbrace{\averagei\ED\Big[ R(\vectorq_i) - R(\centroid{\vectorq}_i) \Big]}_{\textnormal{average variance-effect}}
    -  \underbrace{\EDb{ \averagei\Big[ R(\vectorq_i) - R(\vectorqbar) \Big]}}_{\textnormal{diversity-effect}}
\end{align*}
\end{restatable}

We obtain four terms: the noise, and the {\em effects} of average bias, average variance, and diversity. Note that here we are considering each term including the expectation over $P(\vectorx,y)$.
For losses where Definition \ref{def:gen_bv} holds, Theorem~\ref{the:diversity_effect} reduces to \autoref{the:gen_bvd}, the bias-variance-diversity decomposition.

Similarly to the diversity in \autoref{the:gen_bvd}, the diversity-{\em effect} here has the {\em opposite} sign to the others.  Thus, as diversity-effect grows larger, it {\em reduces} expected risk.
However, the diversity-effect can potentially take {\em negative} values.  This may seem non-intuitive, but makes more sense if we consider the term is simply the difference between the average risk and the ensemble risk. So, when the ensemble performs {\em worse} than the average, this will be negative.\\

In summary, we have the same conclusion as in the previous sub-section: the effect of diversity can be formulated as a hidden dimension in a decomposition of the ensemble risk.  All these terms can be estimated from data---we explore this next.

\newpage

\subsection{\bf Estimating Diversity from Data}

We present illustrative experiments, estimating bias/variance/diversity terms from data.  
We emphasize that we make no claims on the superiority of the centroid combiner over any other, only using it obtain the decomposition.
Further experiments in Appendix~\ref{app:additionalexperiments}. 

\paragraph{Squared loss:} A simple instance of Theorem \ref{the:gen_bvd} is squared loss, $\ell(y,q)=(y-q)^2$, which implies $\qbar=\averagei q_i$, and $\centroid{q}=\ED[q]$. The decomposition at each test point $(\vectorx,y)$ is,
\begin{align}
       \ED\Big[ (\qbar-y)^2 \Big] = 
       \underbrace{\averagei  (\centroid{q}_i-y)^2}_{\textnormal{average~bias}}
        + \underbrace{\averagei \EDb{(q_i-\centroid{q}_i)^2} }_{\textnormal{average~variance}}
        - \underbrace{\ED\Big[\averagei (q_i-\qbar)^2 \Big]}_{\textnormal{diversity}}. \notag
\end{align}
Figure~\ref{fig:californiaBVDexample} shows an experiment with Bagged regression trees, varying ensemble size.
\begin{figure}[h]
    \centering
    \includegraphics[width=0.83\textwidth]{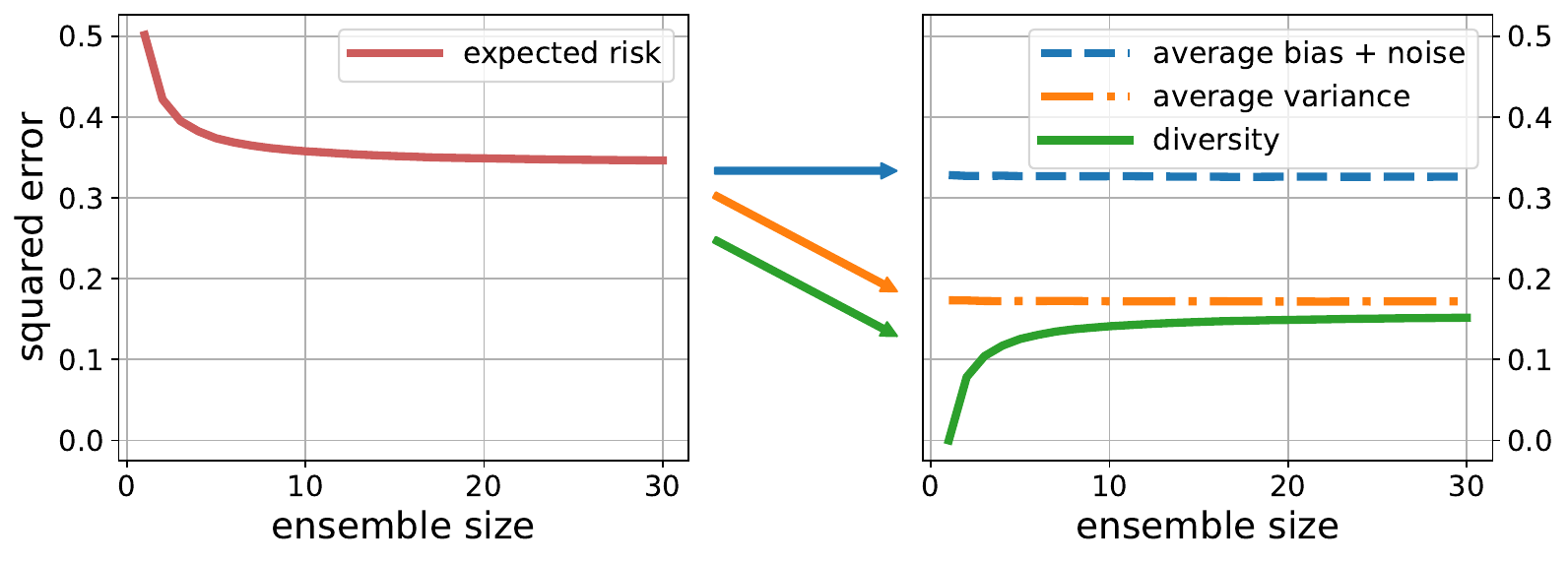}
    \caption{Bagging depth 8 trees, increasing ensemble size (California Housing data).}
    \label{fig:californiaBVDexample}
\end{figure}

Whilst expected risk decreases, the average bias and average variance are constant. This is as we might anticipate, since the individual capacities are constant, it is only the ensemble size, $m$, that we change.
In contrast, diversity {\em increases} with~$m$---subtracting from the expected risk---and the improvement {\em is determined entirely by diversity}.  
This is not so if we vary something other than $m$.  Figure~\ref{fig:californiaBVDexample_depth} fixes $m=10$, varying tree depth---{\em all three components now change}, and we have a {\em bias-variance-diversity} trade-off.
\begin{figure}[ht]
    \centering
    \includegraphics[width=0.83\textwidth]{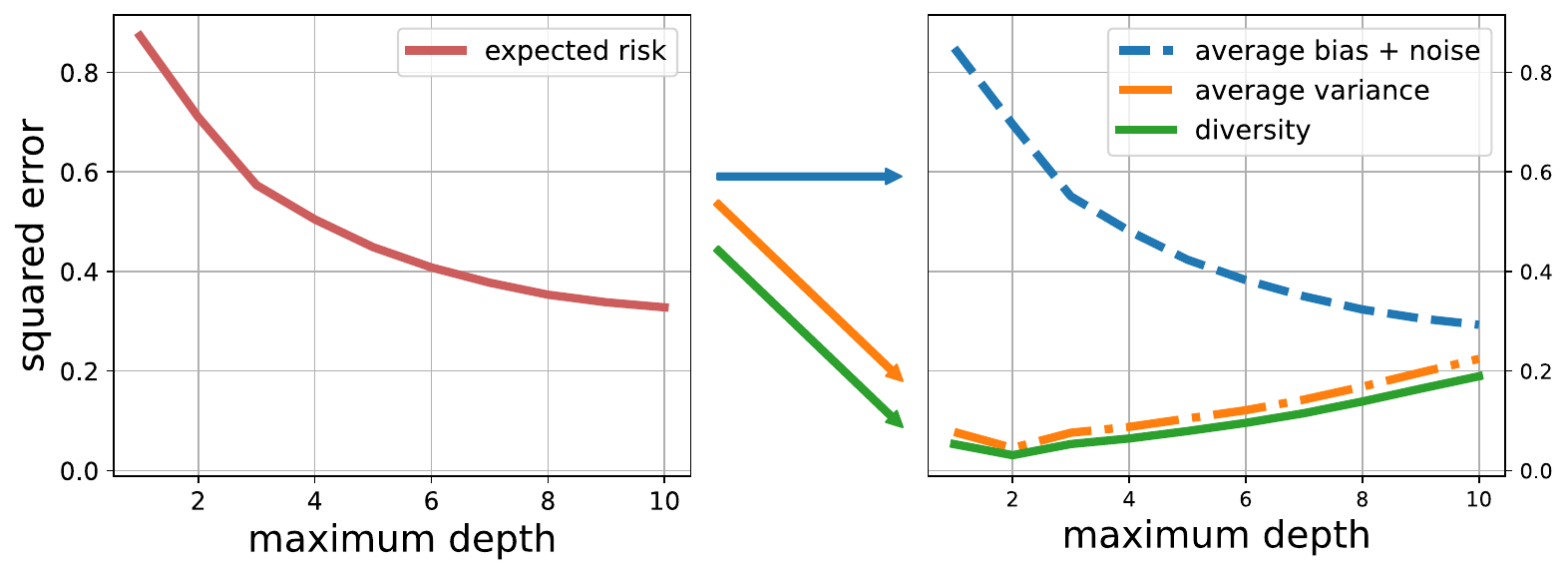}
    \caption{Bagging $m=10$ regression trees, increasing maximum depth.}
    \label{fig:californiaBVDexample_depth}
\end{figure}

\paragraph{Cross-entropy loss:}

We can also consider Theorem~\ref{the:gen_bvd} for the KL-divergence between probability vectors, where the centroid combiner $\vectorqbar$ is a normalized geometric mean.
For a single $({\bf x},{\bf y})$, where $\vectory$ is a one-hot vector,
this gives a decomposition of the cross-entropy, since for any $\vectorq$, we can write  $\KL{\vectory}{\vectorq} = -\vectory \cdot \ln \vectorq$, where we use the convention $\vectory \cdot \ln \vectory = 0$.

\begin{restatable}{corollary}{CrossEntropyDecomposition}\label{the:crossentropydiversity}
Take a test point $({\bf x},{\bf y})$, where $\bf y$ is a one-hot vector.
Consider an ensemble $\{{\bf q}\}_{i=1}^m$ 
where each model predicts a distribution, combined as $\vectorqbar := Z^{-1}\prod_i \vectorq_i^{1/m}$.
We have,
\begin{equation}
    \underbrace{\mystrut{15pt}-\EDb{ \vectory\cdot\ln \overline{\vectorq}} }_{\substack{\textnormal{expected~cross-entropy}}} =
    \underbrace{\mystrut{5pt} -\averagei{ \vectory\cdot\ln\centroid{\vectorq}_i } }_{\textnormal{average bias}} 
        + \underbrace{\mystrut{5pt} \averagei\EDb{ K(\centroid{\vectorq}_i \mid\mid \vectorq_i) } }_{\textnormal{average variance}} 
        - \underbrace{\mystrut{5pt} \EDb{ \averagei K(\overline{\vectorq} \mid\mid \vectorq_i) } }_{\textnormal{diversity}}. \notag
\end{equation}
%
\end{restatable}
 


With ensembles of neural networks, a normalized geometric mean is equivalent to averaging the network logits, a common practice \citep{hinton2015distilling}, with both pros and cons \citep{tassi2022impact}.   It is also known as the `product rule' \citep{kuncheva2014}, and has undergone much scrutiny, showing it can perform either better or worse than the arithmetic mean, depending on various factors \citep{tax1997comparison}. Regardless of performance, the normalized geometric mean is {\em required} to decompose the cross-entropy in this way. In Section \ref{sec:bregman_diversity} we consider the changes in the decomposition if using the arithmetic mean.

In Figure~\ref{fig:decompose_MLPcrossentropy} we compare Bagging MLPs on MNIST, using small networks each with a single layer of 20 hidden nodes, versus larger networks each with 100 hidden nodes.
We observe very similar patterns as we saw with squared loss.

\begin{figure}[h]
    \centering
\includegraphics[width=5.1cm]{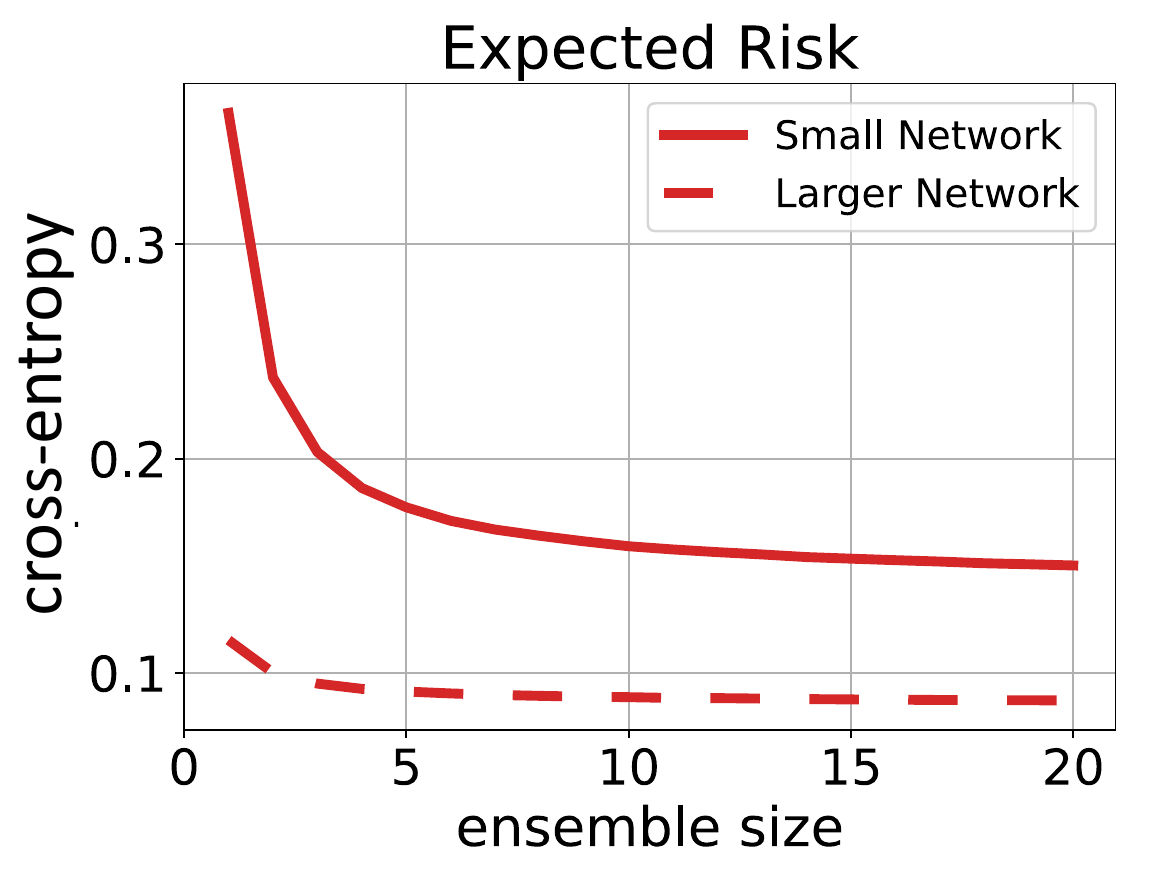}
\includegraphics[width=10cm]{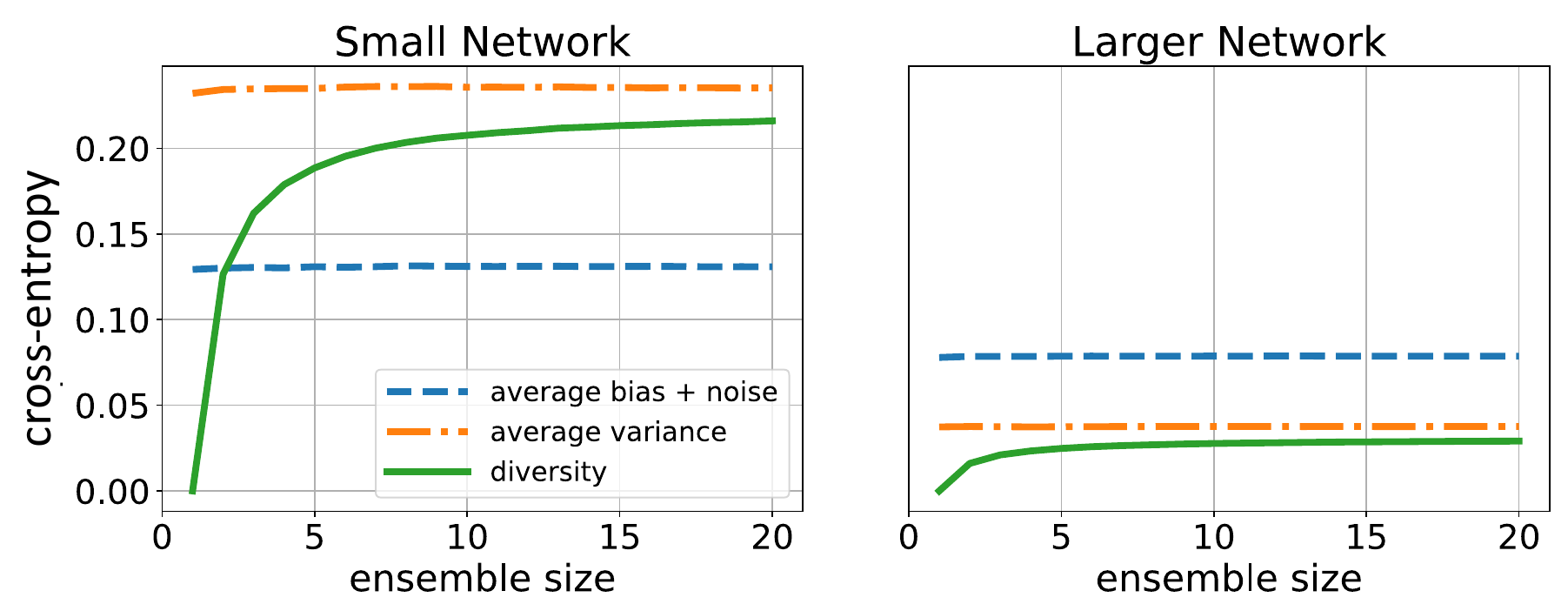}
    \caption{Decomposing expected ensemble cross-entropy for Bagging small/larger MLPs.}
    \label{fig:decompose_MLPcrossentropy}
\end{figure}

Overall, the ensemble of larger networks have performed better: this can again be explained by examining the expected risk components.
We see in the middle panel that the ensemble of smaller networks have a significantly higher diversity (green line).
However, in the far right hand panel, we see the the larger networks have a significantly lower\footnote{The lower variance is counter to the ML folklore that increasing model capacity should also increase variance; but this is consistent with recent observations \citep{yang2020rethinking}} average bias/variance (blue/orange lines).
This gives a significant advantage in the bias/variance/diversity trade-off. %
Overall, the higher diversity among the small networks is outweighed by the more powerful individuals in the ensemble of larger networks.

\newpage
\paragraph{The 0/1 loss and Majority Voting Ensembles:}
One of the most common ensemble schemes is majority voting, assessed by 0/1 loss \citep{ kuncheva2003elusive, didaci2013diversity}.
%
Here, \autoref{the:gen_bvd} does not hold. So, we use the {\em effect} decomposition of \autoref{the:diversity_effect}, and the centroid combiner $\qbar$ is the majority vote.
Using $\mathbb{P}(\qbar\neq y) \defeq\ED[\ell_{0/1}(y,\qbar)]$ for the probability of ensemble error across possible data sets drawn from $D$, at each test point $(\vectorx, y)$ we have:
\begin{align*}
    \mathbb{P}(\qbar \neq y) &=\\
    &\hspace{-1cm}\underbrace{\averagei \Big[ \ell_{0/1}(\centroid{q}_i, y) \Big] }_{\textnormal{average bias-effect}} 
     + \underbrace{\averagei \Big[\mathbb{P}(q_i\neq y) - \ell_{0/1}(\centroid{q}_i, y) \Big]}_{\textnormal{average variance-effect}}
    -  \underbrace{\averagei \Big[\mathbb{P}(q_i\neq y)  - \mathbb{P}(\qbar\neq y) \Big].}_{\textnormal{diversity-effect}}
\end{align*}
Note that the average bias-effect is simply the proportion of the ensemble members whose centroid $\centroid{q}_i$ is incorrect.
The diversity-effect is the difference between the average probability of error, and the ensemble probability of error.
One might be tempted to re-formulate the diversity-effect, looking for a rearrangement that does not involve the label. In fact, for the 0/1 loss, this is impossible: not just with the majority vote, but with {\em any} combiner rule. 

\begin{restatable}[Non-existence of label-independent diversity-effect for 0/1 loss]{theorem}{zeroOneDiversityDoesNotExist}
For the 0/1 loss, using any ensemble combiner rule, the difference between
the average individual risk
and the ensemble risk
is necessarily dependent on the label.
\end{restatable}

Therefore, the effect of diversity in majority voting ensembles is necessarily dependent on the label:  an unavoidable property of the 0/1 loss.
We now compare Bagging and Random Forests on MNIST (10 classes). In the leftmost panel below, we see the Random Forest initially under-performing, but overtaking Bagging as we increase ensemble size.

\begin{figure}[h]
    \centering    \includegraphics[width=\textwidth]{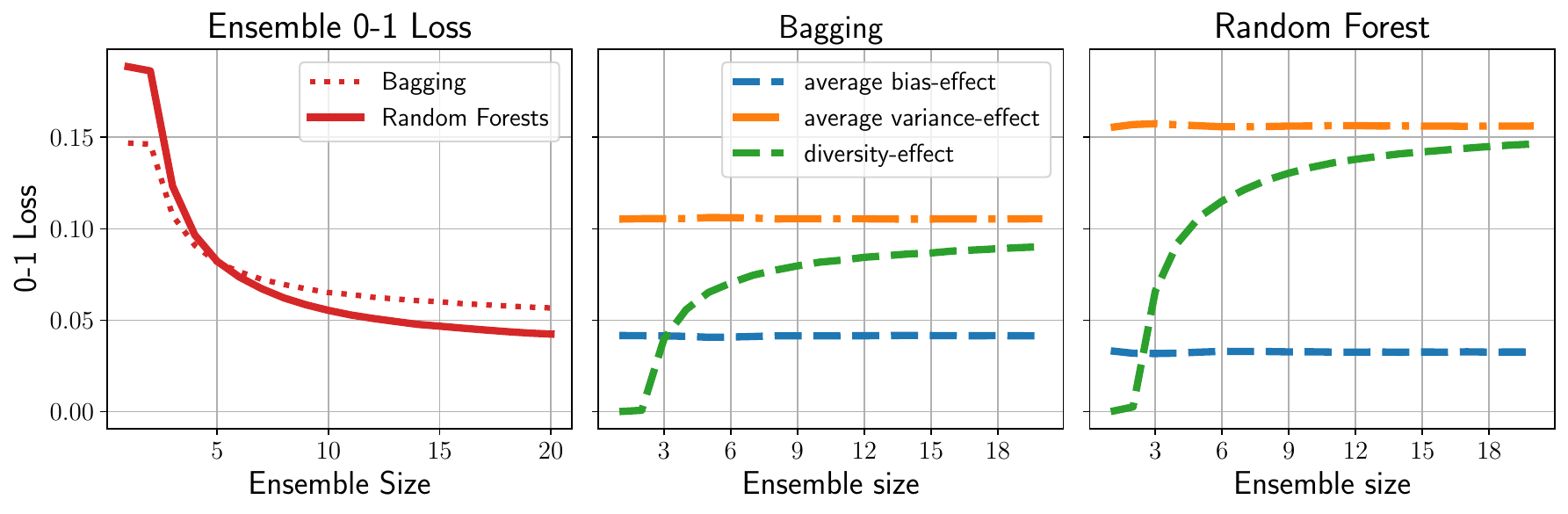}
    \caption{Bias/Variance/Diversity effect for Bagging vs Random Forests.}
    \label{xfig:diversity_effect}
\end{figure}

In the right hand panels, we decompose the expected risk, and see familiar patterns. As we increase $m$, the bias/variance-effects are constant, but diversity-effect increases.
The variance-effect in Random Forests is much higher than that of Bagging, but this is compensated by higher diversity-effect in larger ensembles.  We can also see the average bias-effect is very slightly lower in the Random Forest, which we attribute to the trees being able to avoid noisy features.
%

\newpage

Theorem~\ref{the:diversity_effect} can be extended relatively easily for {\em weighted} voting---as shown in Appendix~\ref{appsubsec:weightedDiversityEffect}.
In Figure~\ref{fig:adaboost_vs_bagging} we plot the ensemble  0/1 risk components for Bagging and two different boosting algorithms, using
decision stumps as base learners. Here we adopt a synthetic dataset used by \citet{mease2008evidence} to analyse boosting algorithms.

\begin{figure}[h]
    \centering
    \includegraphics[width=.95\textwidth]{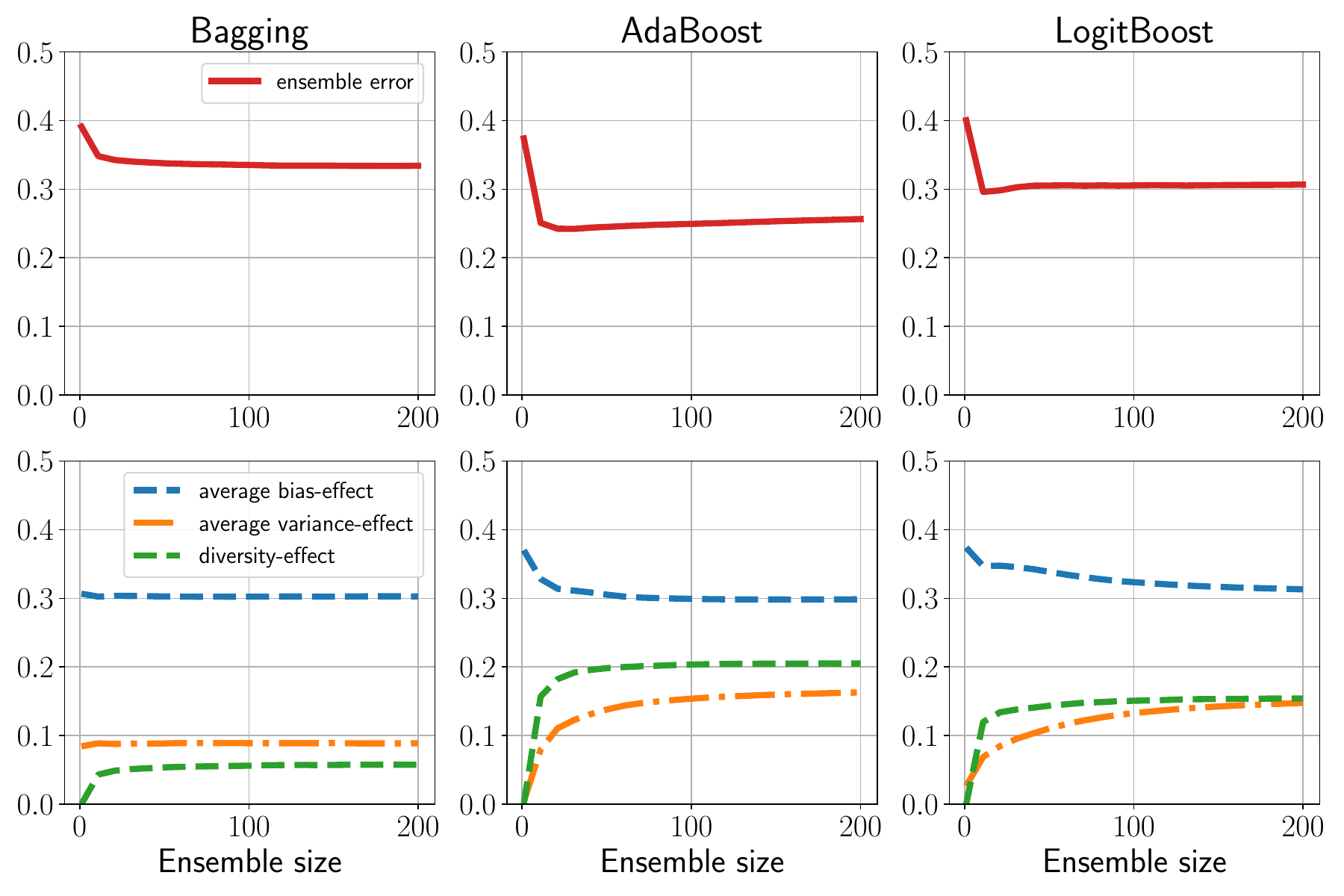}
    \caption{Mease data, ensembling decision stumps.}
    \label{fig:adaboost_vs_bagging}
\end{figure}

We note some interesting differences between the parallel model (Bagging) and the sequential models (LogitBoost/AdaBoost).
As before, the bias/variance are constant for Bagging. However for boosting, they vary with $m$. This is due to the non-homogeneous nature of the individuals, specialising to different parts of the data.
With the 0/1 loss, the diversity-effect can be \emph{greater} than the variance-effect, as opposed to the diversity itself, which is upper-bounded by the average variance (see Section~\ref{sec:bregman_diversity}).
More generally, with Bagging the performance comes from increasing diversity. But, the story for boosting is more subtle, with performance determined by a complex trade-off between the components.

\subsection{Summary}
We presented a framework to understand diversity---explaining it as a measure of model fit, in precisely the same sense as bias/variance, but with the opposite sign. The  expected ensemble risk is then determined by a three-way {\em bias}/{\em variance}/{\em diversity} trade-off.
When an additive bias-variance decomposition is not available, we showed how to instead isolate the {\em effects} of bias/variance and diversity.
In Section \ref{sec:bregman_diversity} we apply this framework with {\em Bregman divergences} \citep{bregman1967relaxation}. This will give us a convenient analytical form, and additional insights: notably, a deeper understanding of the centroid combiner rule.

%% file: 05_BregmanDiversity.tex
\clearpage
\newpage

\section{Understanding
Diversity with Bregman Divergences}
\label{sec:bregman_diversity}

We apply our framework
to the family of {\em Bregman divergences} \citep{bregman1967relaxation},
including many losses as special cases,
and enabling us to derive several interesting properties.

%

\subsection{The Basics of Bregman Divergences}
\newcommand{\targ}{y}
\newcommand{\qmodel}{\ensuremath{q}}

A Bregman divergence is defined in terms of a {\em generator function}, $\phi$.
Let $\phi:\bregmanDomain\rightarrow \mathbb{R}$
be a strictly convex function defined on a convex set $\bregmanDomain\subseteq \mathbb{R}^k$, such that $\phi$ is differentiable on $\relativeinterior(\bregmanDomain)$---the relative interior of $\bregmanDomain$.
The Bregman divergence $B_\phi:\bregmanDomain \times \textrm{ri}(\bregmanDomain) \rightarrow \mathbb{R}_+$ is defined,
\begin{equation}
    \Bregman{\bf \targ}{\vectorq} \defeq \bregman{\bf \targ}{\vectorq},
\end{equation}
where $\langle\cdot,\cdot\rangle$ denotes an inner product, and $\nabla\phi(\vectorq)$ denotes the gradient vector of $\phi$ at $\vectorq$. 
%
%
The choice of $\phi$ leads to many different losses.  With $\generator{q}=q^2$, the gradient vector $\nabla \phi( \vectorq)$ is a scalar derivative $\mathrm{d}\phi(q)/\mathrm{d} q = 2q$, and we recover a squared loss, $\Bregman{\targ}{\qmodel} = (\targ-q)^2$.

\begin{figure}[ht]
\centering
\begin{minipage}[b]{0.46\linewidth}
    \centering
    \includegraphics[width=\textwidth]{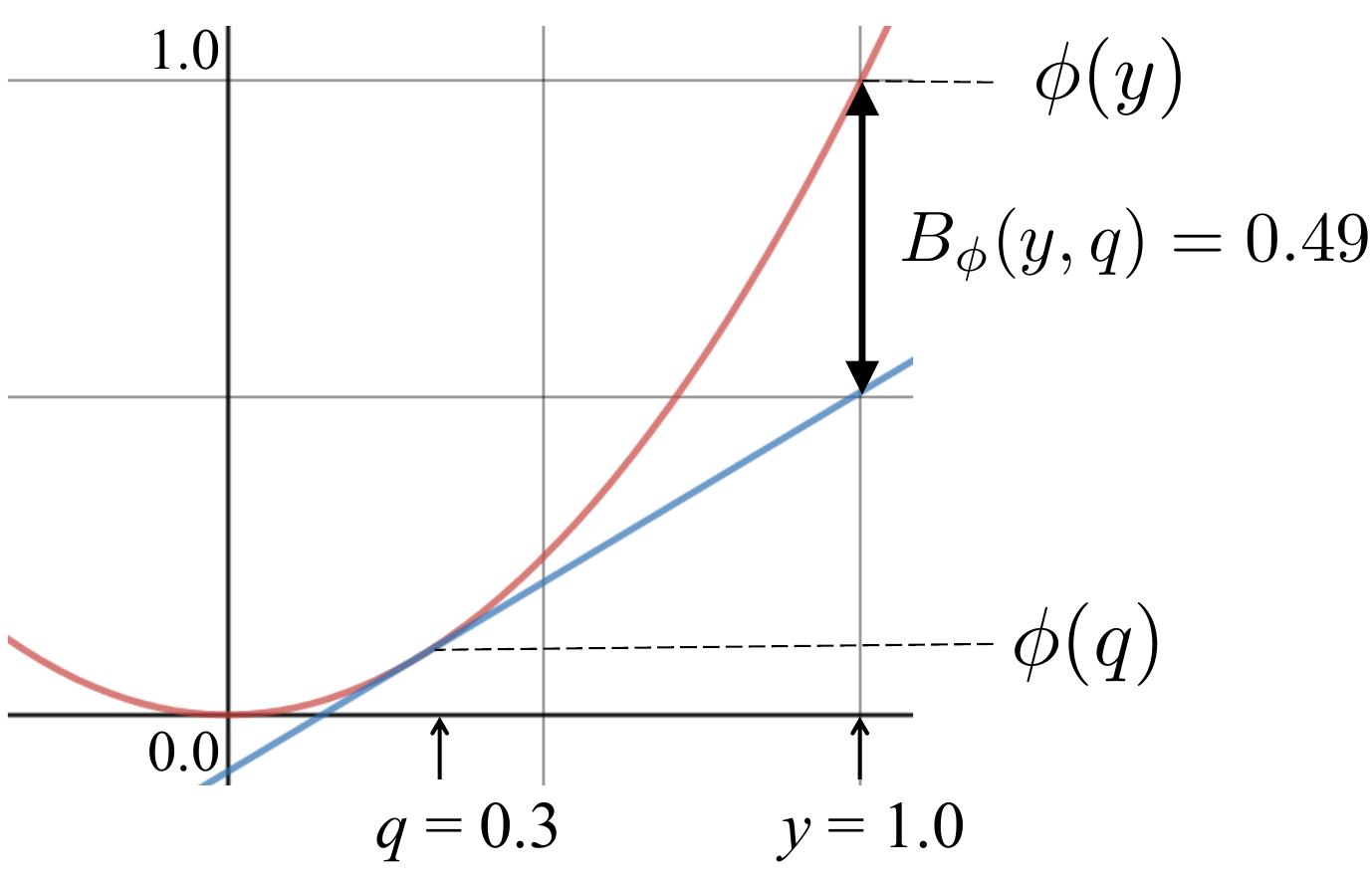}
\end{minipage}
\hspace{1.0cm}
\begin{minipage}[b]{0.4\linewidth}
    \centering
    \footnotesize
    \begin{eqnarray}
        \Bregman{\targ}{\qmodel}&=&\bregman{\targ}{\qmodel}\nonumber\\
            &=& \targ^2-\qmodel^2-\langle 2\qmodel, (\targ-\qmodel)\rangle\nonumber\\
            &=& \targ^2-\qmodel^2 - 2\targ\qmodel + 2\qmodel^2\nonumber\\
            &=& \targ^2 + \qmodel^2 - 2\targ\qmodel\nonumber\\
            &=& (\targ-\qmodel)^2\nonumber
    \end{eqnarray}
    ~\\
    ~\\
\end{minipage}
\caption{Bregman divergence illustrated for the generator $\generator{\qmodel}=\qmodel^2$.  For this example, we have a divergence $\Bregman{\targ}{\qmodel}=(\targ-\qmodel)^2=(1.0-0.3)^2 = 0.49$.}
\label{fig:bregmansquared}
\end{figure}
\noindent Alternatively, we can take a vector $\vectorq \in \mathbb{R}^{k-1}$, for a $k$-class problem. Note, this is not a probability vector summing to one.  It is, however, the minimal description of the distribution, as the $k$th class probability is $1-\sum_c \vectorq^{(c)}$.
With a particular generator (see final row of Table \ref{tab:bregmangenerators}) we recover the the KL-divergence between the distributions in $\mathbb{R}^k$.

\begin{table}[ht]
\centering
\def\arraystretch{1}
\begin{tabular}{@{}lllll@{}}
    \toprule[0.9pt]
    \bf Loss function
       & \bf $\Bregman{\bf y}{\bf q}$ 
       & \bf Generator $\generator{\bf q}$   
       & \bf \shortstack{Domain $\bregmanDomain$}\\ 
    \midrule
        Squared loss
        & $(\targ-\qmodel)^2$ 
        & $\qmodel^2$
        & $\qmodel\in\mathbb{R}$ \\
        Itakura-Saito
        & $\frac{\targ}{\qmodel}-\ln\frac{\targ}{\qmodel}-1$ 
        & $-\ln \qmodel$~~ 
        & $\qmodel\in\mathbb{R}^+$\\
    Poisson loss
        & $\targ \ln \frac{\targ}{\qmodel} - (\targ-\qmodel)$ 
        & $\qmodel\ln \qmodel - \qmodel$ 
        & $\qmodel\in\mathbb{R}^+$\\
    \shortstack{KL-divergence\\~\\~\\~}
        & \shortstack[l]{$\sum \mathbf{\targ}^{(c)} \ln \frac{\mathbf{\targ}^{(c)}}{\mathbf{\qmodel}^{(c)}}~+ $\\$\left(1 - \sum \mathbf{\targ}^{(c)}\right) \ln\frac{1 - \sum  \mathbf{\targ}^{(c)}}{1 - \sum  \mathbf{\qmodel}^{(c)}}$ }
        & \shortstack[l]{\\~\\~\\$\sum \mathbf{q}^{(c)} \ln \mathbf{q}^{(c)}~+$\\ $\left(1 -\sum \mathbf{q}^{(c)}\right) \ln \left(1 - \sum \mathbf{q}^{(c)}\right) $}
        & \shortstack[l]{${\bf \qmodel}\in [0,1]^{k-1}$\\ $s.t. \sum_c \mathbf{q}^{(c)} \leq 1$}\\
        \bottomrule[1.5pt]
\end{tabular}
\caption{Common loss functions and their Bregman generators.}
\label{tab:bregmangenerators}
\end{table}

\subsection{Applying our Framework to the family of Bregman Divergences}

As discussed in Section \ref{sec:unified_diversity}, we require the existence of a bias-variance decomposition in the form of Definition~\ref{def:gen_bv}.  For Bregman divergences, \citet{pfau2013} proved the following:
\begin{align}
    \underbrace{\mystrut{1em}\mathbb{E}_D\Big[ \EXY[\Bregman{\bf Y}{\vectorq}] \Big]}_{\textnormal{expected ensemble risk}} = \Ex\Big[ \underbrace{\mystrut{1em}\mathbb{E}_{\vectorY|\vectorX}[\Bregman{\bf Y}{\vectorYstar}]}_{\textnormal{noise}} + \underbrace{\mystrut{1em}\Bregman{\vectorYstar}{\centroid{\vectorq}}}_{\textnormal{bias}} + \underbrace{\mystrut{1em}\EDb{\Bregman{\centroid{\vectorq}}{\vectorq}}}_{\textnormal{variance}} \Big], \label{eq:bregman_bv}
\end{align}
where the centroid $\centroid{\vectorq}$ is conveniently available in closed-form,
\begin{align}
        \centroid{\vectorq} ~\defeq~ \argmin_{\vectorz\in\bregmanDomain} \mathbb{E}_D\Big[ \Bregman{{\mathbf{z}}}{\vectorq} \Big]
        ~=~
        \left[\nabla\phi\right]^{-1} \Big(\mathbb{E}_D \left[\nabla \generator{\vectorq} \right] \Big).
    \label{eq:main_centroid}
\end{align}
%
%
This matches Definition 2.  If $\phi(q)=q^2$, then $B_\phi(y,q)=(y-q)^2$, and $\centroid{q}=\EDb{q}$, meaning the overall expression is exactly that of \citet{geman1992neural}.
With other generators/losses, the expression corresponds to other bias-variance decompositions, e.g., \citet{heskes1998bias}.
%
From this one result, using the framework presented in Section 4, we can derive Bregman versions of the ambiguity decomposition, the centroid combiner, and the bias-variance-diversity decomposition.

\begin{restatable}[Bregman Ambiguity Decomposition]{theorem}{BregmanAmbiguity}
    \label{the:bregmanambiguity}
    For a label $\vectory\in\bregmanDomain$ and a set of predictions $\vectorq_1, \ldots, \vectorq_m \in \textnormal{ri}(\bregmanDomain)$, combined as ${\bf\qbar} = [\nabla\phi]^{-1}\big(\averagei \nabla\phi( {\bf q}_i) \big)$. Then we have:
    \begin{align}
        \Bregman{\vectory}{{\bf\qbar}} = \averagei 
        \Bregman{\vectory}{{\mathbf{q}}_i} - \averagei \Bregman{{\bf\qbar}}{{\bf q}_i}.\label{eq:bregman_ambiguity}
    \end{align}
\end{restatable}
%


%
%
\begin{definition}[Bregman Centroid Combiner]\label{def:leftcentroid}

The Bregman centroid combiner is the
minimizer of the average divergence 
from all members. For an ensemble $\{\vectorq_i\}_{i=1}^m$, this is

\begin{equation}
\vectorqbar ~\defeq~ \argmin_{\vectorz\in\bregmanDomain} \averagei \Bregman{{\bf z}}{{\bf q}_i} ~=~ \left[\nabla \phi\right]^{-1}\Big(\averagei \nabla \phi( {\bf q}_i ) \Big).
\label{eq:leftcentroid}
\end{equation}
\end{definition}


\begin{restatable}[Bregman Bias-Variance-Diversity decomposition]{theorem}{bvdd}~\label{the:bvdd}
~\\
    For an ensemble of models $\{\vectorq_i\}_{i=1}^m$, let $\centroid{\vectorq}_i$ be the left Bregman centroid of model $\vectorq_i$, i.e., $\centroid{\vectorq}_i \defeq [\nabla \phi]^{-1}\left(\EDb{\nabla\phi(\vectorq_i)} \right)$, and define the ensemble output $\vectorqbar\defeq\gradinv{\averagei\nabla\generator{\vectorq_i} }$. Then we have the decomposition
    \begin{align*}
    \EDb{ \EXYb{\Bregman{\vectorY}{\vectorqbar}}}  &=\\
    %
    %
    &\hspace{-3.5cm}\mathbb{E}_{\bf X}\Bigg[
        \underbrace{\mystrut{15pt}\mathbb{E}_{{\bf Y}|{\bf X}}\Bregman{\vectorY}{\vectorYstar} }_{\textnormal{noise}}
        +\underbrace{\mystrut{15pt}\averagei \Bregman{\vectorYstar }{\centroid{\vectorq}_i}}_{\textnormal{average bias}} 
        + \underbrace{\mystrut{15pt}\averagei\EDb{\Bregman{\centroid{\vectorq}_i}{\vectorq_i} }}_{\textnormal{average variance}} 
        - \underbrace{\mystrut{15pt}\mathbb{E}_D\Big[\averagei \Bregman{\vectorqbar}{\vectorq_i }}_{\textnormal{diversity}}\Big]\Bigg], \notag
    \end{align*}
    where $\vectorYstar = \mathbb{E}_{\vectorY | \mathbf{X}} \left[ \vectorY \right] $.
\end{restatable}

Examples for different losses (i.e., different Bregman generators) are shown in Table \ref{bigtableofmeasures}.
One point that this makes clear is that {\em the mathematical formulation of diversity is specific to the loss function}. 

\newpage

\begin{landscape}
\begin{table}[htbp]\renewcommand{\arraystretch}{3}
\setlength\extrarowheight{2pt} 

\begin{center}\small
\begin{tabular}{@{}llll@{}}
\toprule[1.5pt]
        \bf \shortstack[l]{Expected Ensemble Loss}
        &\bf Average Bias
        &\bf Average Variance
        &\bf Diversity
    \\ \midrule
        \shortstack[l]{\bf Squared
        \\${\EDb{(\qbar-y)^2}}$}
        & \mystrut{.6cm}$\displaystyle\averagei {(\centroid{q}_i - y)^2}$
        & $\displaystyle\averagei \EDb{(q_i - \centroid{q}_i)^2}$
        &$\displaystyle\mathbb{E}_D\Big[~ \averagei \Big( q_i - \qbar \Big)^2 ~\Big]$\\
        \shortstack[l]{~\\\bf KL-divergence (Bernoulli) \\
        ${\EDb{y\ln\frac{y}{\qbar}+(1{-}y)\ln\frac{1{-}y}{1{-}\qbar}}}$}
        & \mystrut{0.6cm}$\frac{1}{m} \sum\limits_{i=1}^m y\ln \frac{y}{\centroid{q}_i} + (1-y)\ln \frac{1-y}{1-\centroid{q}_i}$
        & \mystrut{0.6cm}$\frac{1}{m} \sum\limits_{i=1}^m \ED \left[ \centroid{q}_i\ln \frac{\centroid{q}_i}{q_i} + (1-\centroid{q}_i)\ln \frac{1-\centroid{q}_i}{1-q_i}\right]$
        & \mystrut{0.6cm}$\EDb{\frac{1}{m}\sum\limits_{i=1}^m \qbar\ln \frac{\qbar}{q_i} + (1-\qbar)\ln \frac{1-\qbar}{1-q_i}}$\\   
        \shortstack[l]{~\\ \bf KL-divergence (Multinoulli) \\
        ${\EDb{\KL{\vectory}{\vectorqbar}}}$}
       & \mystrut{0.7cm}$\displaystyle\averagei {\KL{\vectory}{\centroid{\vectorq}_i}}$
        & $\displaystyle\averagei\EDb{\KL{\centroid{\vectorq}_i}{\vectorq_i}}$ 
        & $\displaystyle \EDb{\averagei \KL{\vectorqbar}{\vectorq_i}}$ \\ 
        \shortstack[l]{ \bf Itakura-Saito
        \\ ${\EDb{\frac{y}{q} - \ln \frac{y}{\qbar} - 1 }}$}
        &  \mystrut{0.7cm}$\displaystyle\averagei \frac{y}{\centroid{q}_i} - \ln \frac{y}{\centroid{q}_i} - 1 $
        & $\displaystyle\averagei \EDb{\frac{\centroid{q}_i}{q_i} - \ln \frac{\centroid{q}_i}{q_i} - 1 }$
        & $\displaystyle\mathbb{E}_D \left[ \averagei\Bigg( \frac{\qbar}{\qmodel_i}-\ln\frac{\qbar}{\qmodel_i}-1\Bigg)\right]$\\
        \shortstack[l]{\bf Poisson
        \\${\EDb{y \ln \frac{y}{\qbar} - (y-\qbar)}}$}
        &  \mystrut{0.7cm}$\displaystyle\averagei \left[y \ln \frac{y}{\centroid{q}_i} - (y-\centroid{q}_i)\right]$
        & 
        $\displaystyle\averagei\Big[ \EDb{q_i} - \centroid{q}_i \Big]$
        & $\displaystyle\mathbb{E}_D\Bigg[ \averagei q_i ~-~ \prod_{i=1}^m q_i^{1/m} \Bigg]$\\
    \bottomrule[1.5pt]
\end{tabular}
\end{center}
\caption{Bias, variance, and diversity under different Bregman divergences. In all cases, the expectation over $p(\vectorx)$ is omitted and the expressions given are for a single point $(\vectorx,\vectory)$.
See Table~\ref{main:tab:bregmancombiners} for the definitions of the centroid combiners $\qbar$.}
\label{bigtableofmeasures}
\end{table} 
\end{landscape}
\pagestyle{plain}

\subsection{Properties of the Bregman Centroid Combiner}

The centroid combiner is the result of a constraint, requiring the ensemble combination $\vectorqbar$ to take the same analytical form as the centroid $\centroid{\vectorq}$, as found in the bias-variance decomposition for the relevant loss.
For Bregman divergences, this form (Definition~\ref{def:leftcentroid}) is known in the information geometry literature: as a {\em left Bregman centroid} \citep{nielsen2009sided}.   In general these turn out to be
{\em quasi-arithmetic} means, and include several well-known ensemble combiner rules, some shown in Table~\ref{main:tab:bregmancombiners}.
\begin{table}[ht]
\centering
\def\arraystretch{1.2}
\begin{tabular}{@{}lll@{}}
        \toprule[1.5pt]
        \bf Loss function
        & \bf Centroid Combiner
        & \bf Name\\ \midrule
        Squared loss 
        & $\averagei q_i$
        & Arithmetic mean\\[2ex]
        Poisson regression loss 
        & $\prod_{i=1}^m q_i^{\frac{1}{m}}$
        & Geometric mean\\[2ex]
        \shortstack{KL-divergence \\ \hspace{1em}} & 
        $Z^{-1}\prod_{i=1}^m\left(\vectorq_i^{(c)}\right)^{\frac{1}{m}}$
        & Normalised geometric mean\\[2ex]
        Itakura-Saito loss & 
        $1 \big{/} \Big(\averagei \frac{1}{q_i}\Big)$
        & Harmonic mean\\[1ex]
        \bottomrule[1.5pt]
\end{tabular}
    \caption{Centroid combiners (i.e., left Bregman centroid of the ensemble) for various losses.}
\label{main:tab:bregmancombiners}
\end{table}

%

{\noindent\bf Ensemble averaging in a dual coordinate system:} The centroid combiner can be understood as an ensemble averaging operation,
{\em in a new coordinate system}. The mapping between coordinate systems is defined by the gradient of the Bregman generator with respect to its argument, $\nabla\phi(\vectorq)$.  This is illustrated  in Figure~\ref{fig:spaces} for the KL-divergence.

\begin{figure}[ht]
    \centering

    \begin{tikzpicture}[scale=1.35,
    thick,
    >=stealth',
    dot/.style = {
      draw,
      fill = white,
      circle,
      inner sep = 0pt,
      minimum size = 4pt
    }
    ]
    \input{images/Section5/ensembleAveragingGeometry-Fig10.tex}
    \end{tikzpicture}
    \caption{Ensemble averaging in the geometry defined by the KL-divergence.} 
    \label{fig:spaces}
    
\end{figure}
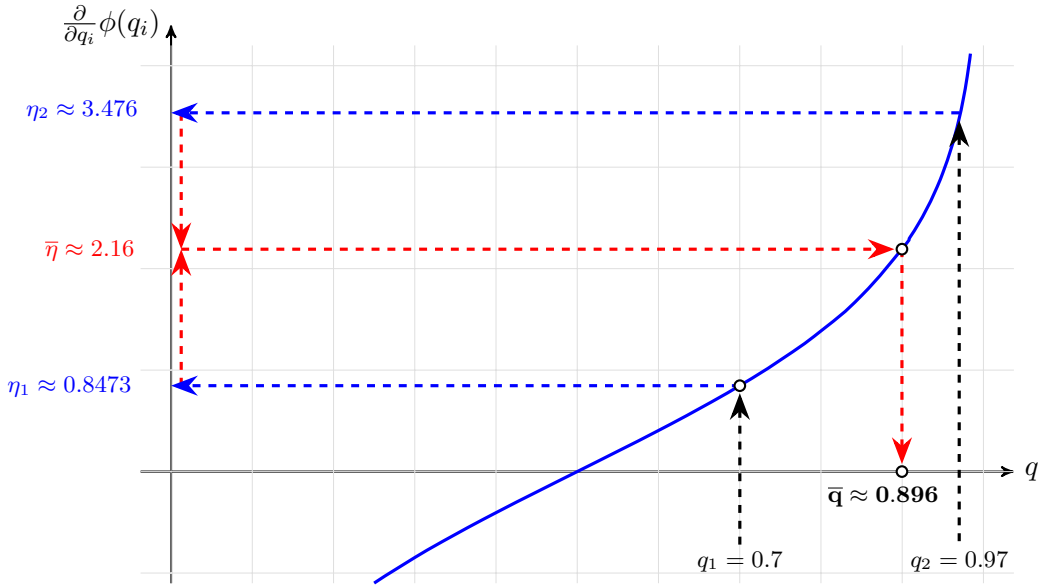

We use notation $q$ for the \emph{primal} coordinate system, and $\eta$ for the \emph{dual} coordinate system. In this simple illustration we are predicting a single probability $p\in (0,1)$.
The primal-dual mapping is the gradient $\eta_i=\frac{\partial}{\partial q_i}\phi(q_i)=\ln\frac{q_i}{1-q_i}$, plotted as the blue curve.  Two points in the primal $\{q_1=0.7,q_2=0.97\}$ are mapped to the dual $\{\eta_1\approx 0.8473 ,\eta_2\approx 3.476\}$, then combined via arithmetic mean ($\overline{\eta}\approx 2.16$), and finally mapped back by the inverse operation $\overline{q}=\textrm{exp}(\overline{\eta})/(1+\textrm{exp}(\overline{\eta}))\approx 0.896$.
The centroid combiner is therefore an {\em arithmetic mean ensemble in the dual coordinate system}, which is equivalent to
the left Bregman centroid of the models in the {\em primal} coordinate system. 
An equivalent definition was considered by \citet{gupta2022ensembling} under the assumption of i.i.d. models. Our analysis both complements and extends this by removing the i.i.d. assumption, and more fully characterising the properties of ensembles using this combination rule.\\

{\noindent \bf Example for KL-divergence of probability estimates:} In the case of KL divergence, the centroid combiner is a normalised\footnote{Note that the Bregman centroid is only a normalized distribution with $\phi$ as in Table 2---see Appendix~\ref{app:parameterEncoding}.} geometric mean in the primal coordinate system, i.e., the probability simplex.  Figure~\ref{fig:centralmodelimage} shows this for the 3-class case. 
    
\begin{figure}[ht]
\centering
\includegraphics[width=\textwidth]{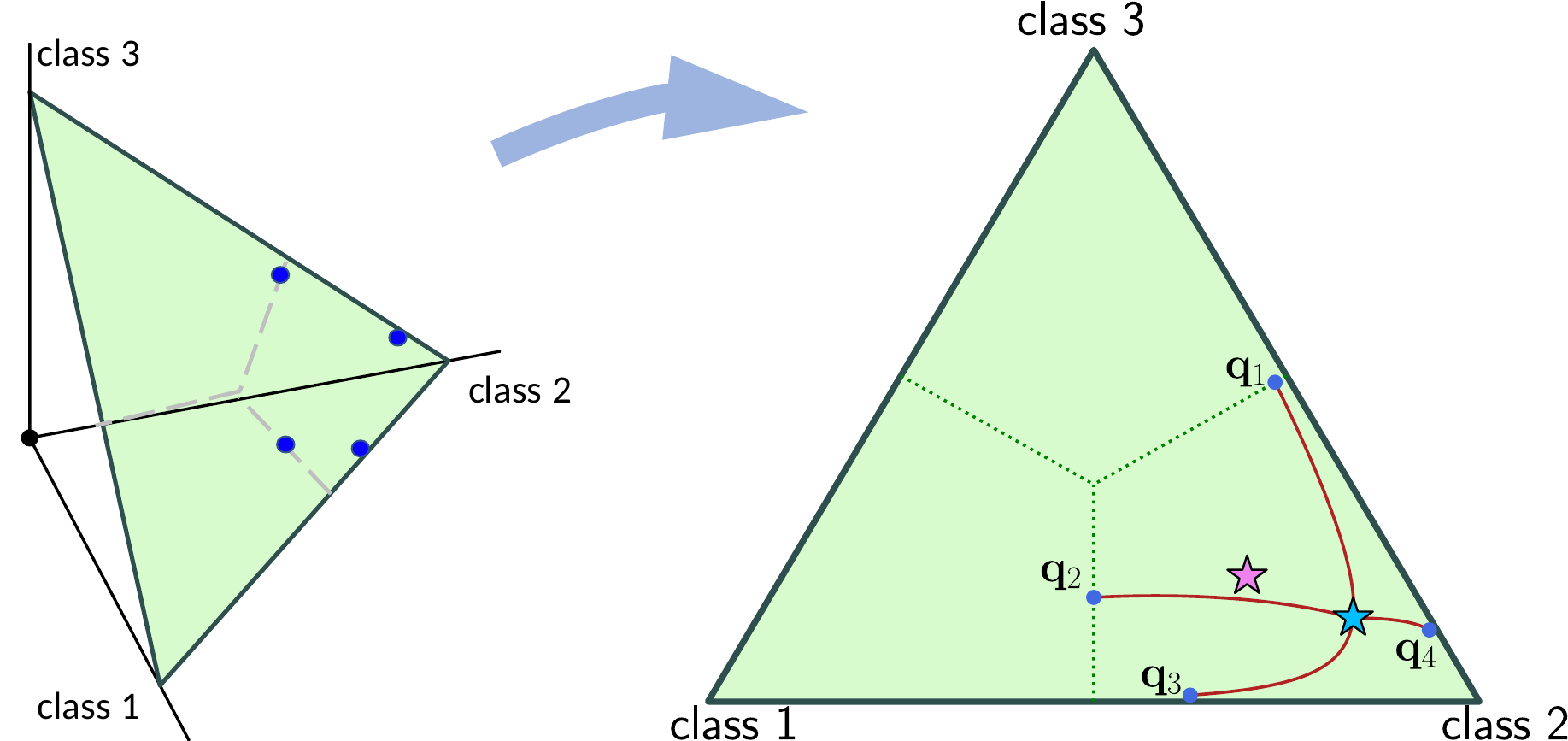}
\caption{Combining $m=4$ predictions in the probability simplex. }
\label{fig:centralmodelimage}
\end{figure}

The centroid combiner is shown as the blue star (also connected to the 4 individual points)
and the {\em arithmetic} mean as the pink star, which is the centre of mass for the points.
{\noindent For squared loss the resulting Bregman geometry is Euclidean, meaning the centre of mass and the centroid are the same.  However, other divergences define a non-Euclidean geometry, meaning the centroid and centre of mass are different. }
Note that points are connected in the simplex not by straight lines, but by {\em dual geodesics} defined by the choice of $\phi$~\citep{nielsen2009sided}.
A different $\phi$ (and therefore a different loss) would provide a different
geometry, and imply a different ensemble combination.\\


\newpage

\newpage
\subsection{Properties of the Decomposition}
\label{sec:furtherproperties}

We explore further properties of the bias-variance-diversity decomposition. 

\subsubsection{Homogeneous vs Heterogeneous Ensembles}

An ensemble comprises a set of models $\{\vectorq_i\}_{i=1}^m$, but can also be considered as `one big model' in its own right.  Thus, we can discuss the bias/variance of $\vectorqbar$ itself.
%
%
These bias/variance terms can be related to those of the individual models, using the theorem below.
\begin{restatable}{theorem}{ensbiasvariance}
The ensemble bias and ensemble variance can be re-written as a function of the individual models:
\begin{align}
\underbrace{\mystrut{15pt} \Bregman{\vectory}{\centroid{\vectorqbar}} }_{\textnormal{ensemble bias}} &~=~ \underbrace{\mystrut{15pt}\averagei\Bregman{\vectory}{\centroid{\vectorq}_i}}_{\textnormal{average bias}} ~-~ \Delta, \label{eq:ensemblebias}\\
\underbrace{\mystrut{15pt}\EDb{\Bregman{\centroid{\vectorqbar}}{\vectorqbar}}}_{\textnormal{ensemble variance}} &~=~ \Delta ~+~ \underbrace{\mystrut{15pt}\averagei\EDb{\Bregman{\centroid{\vectorq}_i}{\vectorq_i}}}_{\textnormal{average~variance}} - \underbrace{\mystrut{15pt}\EDb{\averagei\Bregman{\vectorqbar}{\vectorq_i}}}_{\textnormal{diversity}},
    \label{eq:ensemblevariance}
\end{align}
where the common term is $\Delta = \averagei \Bregman{\centroid{\vectorqbar}}{\centroid{\vectorq}_i}$.
\end{restatable}

The interesting term here is $\Delta$, which we refer to as the model ``disparity''. This accounts for  diversity among the model {\em families} in the ensemble.
If the individual models are all the same family, (e.g., all neural nets of a particular architecture), then the only differences in their predictions will be due to the random variable $D$.  In this case their centroids will be the same $\centroid{\vectorq}_i=\centroid{\vectorq}_j=\centroid{\vectorqbar}$, and the disparity will be zero. 
If instead they have different inductive biases, the centroids are not necessarily the same, and the disparity may be non-zero. These are usually referred to as {\em homogeneous} vs. {\em heterogeneous} ensembles.\\

{\noindent\bf Bias reduction?}
In a homogeneous ensemble, $\Delta=0$, and the ensemble bias is equal to the average, i.e., no reduction in bias, relative to the individual models.  In a heterogeneous ensemble with sufficiently different model families, $\Delta>0$, and the bias is {\em always reduced}.\\

{\noindent\bf Variance reduction?} Ensembles are often referred to as `variance reduction' methods. 
We see in Equation~\eqref{eq:ensemblevariance} that for any Bregman divergence, for a homogeneous ensemble ($\Delta=0$), the ensemble variance is guaranteed to be less than the average, i.e., we retain the variance reduction property, {\em provided that we use the centroid combiner rule.}  From the same perspective, we see the diversity is upper-bounded by the average variance.   The amount by which the ensemble variance is reduced is exactly the diversity of the ensemble.  In a heterogeneous ensemble, the story is less simple, as it has both the addition of the disparity, and the subtraction of the diversity.\\

The properties above apply for {\em any} Bregman divergence. In practice, we will have a particular choice for a learning problem---we now discuss specific properties of two ubiquitous losses: the cross-entropy, and squared loss.

\newpage

\subsubsection{Cross Entropy Loss: Averaging estimates of class probabilities}
When combining class probability estimates, a very popular strategy is to take their arithmetic mean, e.g., \citet{lakshminarayanan2017simple}, but, if we use the cross-entropy, this is {\em not} the centroid combiner. 
We might wonder what effect this has.  The proposition below demonstrates that the cross-entropy loss of the ensemble is still \emph{guaranteed to be less than the average loss of its members}, but the ambiguity becomes dependent on the label.

\begin{proposition}
Assume a true probability vector, $\vectory\in \mathbb{R}^{k}$,
and a set of models $\{\vectorq_i\}_{i=1}^m$ combined by averaging, i.e., $\vectorq^\dagger=\averagei \vectorq_i$, then the cross-entropy loss is
\begin{align}
   \underbrace{\mystrut{2.em}-{ \vectory\cdot\ln {\vectorq^\dagger} }}_{\textnormal{ensemble cross-entropy}} ~=~ \underbrace{\mystrut{2.em}-\averagei \vectory\cdot\ln{\vectorq_i}}_{\textnormal{average cross-entropy}} ~-~ \underbrace{\mystrut{2.em}\sum_{c=1}^{k} {\vectory}^{(c)} \ln \dfrac{\frac{1}{m}\sum_{j=1}^{m} {\vectorq}_j^{(c)}} {\left(\prod_{i=1}^{m} {\vectorq}_i^{(c)}\right)^{\frac{1}{m}} }}_{\textnormal{ambiguity (label-dependent)}},\label{eq:arithmetic_ambiguity}
\end{align}
where the second term is non-negative, thus the ensemble loss is guaranteed less than or equal to the average individual loss.
\label{prop:averagedprobability}
\end{proposition}

This property can be observed without the framework we have presented thus far, by taking the difference between the $-{ \vectory\cdot\ln {\vectorq^\dagger} }$ and $-\averagei \vectory\cdot\ln{\vectorq_i}$ terms.
In fact, for the case of two classes, this was observed independently by \citet{ivascu2021}.
%
%
Using the general case above, we can combine it with the methodology from Section~\ref{sec:unified_diversity}, to obtain a result analogous to the bias-variance-diversity decomposition.

\begin{restatable}[Diversity for Averaged Probabilities  is label-dependent]{proposition}{klDependency}~\\
Let $\vectorq^\dagger = \averagei \vectorq_i$, with $\vectorq_i\in [0,1]^k$. The expected cross-entropy admits the decomposition:

\begin{eqnarray}
    -\EDb{ {\vectory\cdot\ln{\vectorq^\dagger}  }} &=&\notag\\
        &&\hspace{-3.3cm} \underbrace{\mystrut{2.em}-\averagei\vectory\cdot\ln {\centroid{\vectorq}_i}}_{\textnormal{average bias}}
        \,+~\, \underbrace{\mystrut{2.em}\averagei\EDb{\KL{\centroid{\vectorq}_i}{\vectorq_i} }}_{\textnormal{average variance}} - \underbrace{\mystrut{2.em}\EDb{\sum_{c=1}^{k} {\vectory}^{(c)} \ln \dfrac{\frac{1}{m}\sum_{j=1}^{m} {\vectorq}_j^{(c)}} {\left(\prod_{i=1}^{m} {\vectorq}_i^{(c)}\right)^{\frac{1}{m}} } }}_{\textnormal{dependency}}.\notag
\end{eqnarray}

\end{restatable}

This is very similar to \autoref{the:crossentropydiversity}, except that the final term here is dependent on the label $\vectory$.  To distinguish these, we avoid using the name ``diversity'', and instead refer to it as a ``dependency'' term.  

\cite{gupta2022ensembling} also studied properties of ensemble bias/variance with an arithmetic mean combiner, showing that (under an i.i.d. model assumption), the ensemble variance is always reduced.   At the same time, they raised a concern, that this may potentially increase the ensemble bias (above the average bias), dependent on the label distribution. 
Our proposition adds insight: the overall expected loss will {\em always} be less than the average.  Thus, even if there is an increase in ensemble bias, it is always more than compensated by the reduction in ensemble variance, leading to lower overall expected loss.

\newpage

\subsubsection{Relation to the Bias-Variance-Covariance Decomposition}\label{subsubsec:squaredbvc}

For the case of squared loss, our decomposition
can be contrasted with the bias-variance-covariance decomposition of \citet{ueda1996generalization}, which states, for a point $(\vectorx,y)$:
\begin{align}
     \EDb{~(\qbar-y)^2~} &= \\
     &\hspace{-2.5cm} \underbrace{\mystrut{15pt} (\EDb{\qbar}-y)^2}_{\textnormal{bias}(\qbar)} + \underbrace{\mystrut{15pt} \frac{1}{m}\averagei\ED\Big[{(q_i-\EDb{q_i})^2}\Big]}_{1/m ~\times~ \overline{\textnormal{variance}}} +  \underbrace{\mystrut{15pt} \frac{1}{m^2}\sum_{i,j}\ED\Big[{(q_i-\EDb{q_i})(q_j-\EDb{q_j})\Big]}}_{(1-1/m)~\times~\overline{\textnormal{covariance}}}. \notag 
\end{align}
Comparing this to our Theorem~\ref{the:bvdd}, a difference can be seen in the bias components. Ours is the {\em average individual} bias, whereas Ueda \& Nakano's is the {\em ensemble} bias:
\begin{eqnarray}
\bias &=& \averagei (\EDb{q_i}-y)^2, \\
\textnormal{bias}(\qbar) &=& (\EDb{\qbar}-y)^2.
\end{eqnarray}
Ueda \& Nakano observed a re-writing of their term: $(\EDb{\qbar}-y)^2=(\averagei[ \EDb{q_i}-y])^2$, and described this as {\em ``the square of the average biases''}.
%
However, we remind the reader that {\em the square is an artefact of using the squared loss function}.    This square is not present in generalised forms of the bias-variance decomposition---see Appendix~\ref{app:squaredbias}.
Thus, $(\EDb{\qbar}-y)^2$ should be referred to as simply the ``bias of the ensemble''.
The {\em difference} between $\bias$ and $\textnormal{bias}(\qbar)$ is itself an interesting quantity---if we instantiate \eqref{eq:ensemblebias} for squared loss, we see $\textnormal{bias}(\qbar)=\bias-\Delta$. 
This shows that their term is in fact made up of two components: the average of the individual biases, and  the {\em disparity} term introduced above.
Of course, if the models are {\em homogeneous}, i.e., of the same family, then the disparity is zero, and $\textnormal{bias}(\qbar)=\bias$.

Ueda \& Nakano's decomposition relies on a property of linear combinations of random variables: 
%
$Var(aX_1+bX_2)=a^2Var(X_1)+b^2Var(X_2)+2abCov(X_1,X_2)$,
where $X_1,X_2$ represent two model outputs and $aX_1+bX_2$ represents the ensemble combination.
When the ensemble combination rule is {\em non-linear} (i.e., not a simple arithmetic mean) then this property (and hence this as a route to understand diversity) no longer applies.

It is notable that the covariance can be either positive or negative.  Our diversity term, however, is always non-negative, growing with more disagreement around the ensemble decision. The covariance is a  fundamentally {\em pairwise} computation---it is likely that this form inspired the many published {\em pairwise} diversity measures \citep{kuncheva2014}. 
Diversity in the Bregman case is written in a {\em non-pairwise} manner---the expected average deviation around the ensemble prediction.  We conjecture that this term cannot be expressed as solely pairwise operations, implying that pairwise measures may be fundamentally limited.

\subsection{Summary}
We applied the framework developed in Section \ref{sec:unified_diversity} to the broad family of Bregman divergences. These have a particularly convenient analytical form, which enabled us to derive several interesting properties. 

%% file: images/Section5/ensembleAveragingGeometry-Fig10.tex
  \coordinate (O) at (0,0);
  \draw[->] (-0.3,0) -- (8.3,0) coordinate[label = {right:$q$}] (xmax);
  \draw[->] (0,-1.1) -- (0,4.4) coordinate[label = {left:$\frac{\partial}{\partial q_i}\phi(q_i)$}] (ymax);
  \draw[ultra thin, gray!25] (-.3,-1.1) grid [xstep=.8,ystep=1.] (8.3,4.2); 
  \path[name path=firstq] (5.6,-0.5) -- (5.6,4.0);
  \path[name path=secondq] (7.76,-0.5) -- (7.76,4.0);
  \path[name path=y] plot[smooth] coordinates {    (2,-1.099)
    (2.4,-.847)
    (2.8,-.619)
    (3.2,-.405)
    (3.6,-.201)
    (4,.000)
    (4.4,.201)
    (4.8,.405)
    (5.2,.619)
    (5.6,.847)
    (6,1.099)
    (6.4,1.386)
    (6.8,1.735)
    (7.2,2.197)
    (7.6,2.944)
    (7.84,3.892)};
  \draw[blue, very thick] plot[smooth, tension=.85] coordinates {
    (2,-1.099)
    (2.4,-.847)
    (2.8,-.619)
    (3.2,-.405)
    (3.6,-.201)
    (4,.000)
    (4.4,.201)
    (4.8,.405)
    (5.2,.619)
    (5.6,.847)
    (6,1.099)
    (6.4,1.386)
    (6.8,1.735)
    (7.2,2.197)
    (7.28,2.314)
    (7.36,2.442)
    (7.44,2.587)
    (7.52,2.752)
    (7.6,2.944)
    (7.68,3.178)
    (7.72,3.317)
    (7.76,3.476)
    (7.8,3.664)
    (7.84,3.892)
    (7.872,4.119) };
  \scope[name intersections = {of = firstq and y, name = i}]
  
    \scope[name intersections = {of = secondq and y, name = j}]
    \coordinate(centroid) at ($(i-1)!0.5!(j-1)$);
    \path[name path=qbar] (0,0 |- centroid) -- (8.2,0 |- centroid);
         \scope[name intersections = {of = qbar and y, name = k}]

        \draw[dashed, very thick, red, -{Stealth[scale=1.2]}, shorten >=.1cm] (0.1,0 |- centroid)  -- (k-1 |- centroid) node[
                              label = {[xshift=-10.8cm, yshift=-0.45cm]above:\footnotesize$\overline{\eta} \approx 2.16$}] {};
                              
        \draw[dashed, very thick, red, -{Stealth[scale=1.2]}, shorten >=.1cm]  (k-1 |- centroid) -- (k-1 |- 0,0)   node[label ={[black, xshift=-.25cm, yshift=.1cm] below:\footnotesize$\bf \overline{q} \approx 0.896$}] {};
        
        \path  (k-1 |- 0,0) node[dot] {};
        \path (k-1) node[dot] (i-2) {} -- (i-2 |- i-2);
             
        \draw[dashed, very thick, red, -{Stealth[scale=1.2]}]  (0.1,0 |- j-1) -- (0.1,0 |- centroid) {};
        \draw[dashed, very thick, red, -{Stealth[scale=1.2]}]  (0.1,0 |- i-1) -- (0.1,0 |- centroid) {};
       
        \endscope
    \draw  [dashed, very thick, shorten <=.1cm, {Stealth[scale=1.2]}-] (i-1) -- (i-1 |- 0,-.75) node[
                              label = {[yshift=.25cm]below:\footnotesize$q_1 = 0.7$}] {};
    \draw [dashed, very thick, blue, -{Stealth[scale=1.2]}]  (i-1) -- (0,0 |- i-1) node[
                              label = {[xshift=-1.35cm, yshift=-0.45cm]above:\footnotesize${\eta_1} \approx 0.8473$}] {};
                              
    \path (i-1) node[dot] (i-2) {} -- (i-2 |- i-2);
    
    \draw   [dashed, very thick, shorten <=.1cm, {Stealth[scale=1.2]}-](j-1) -- (j-1 |- -0,-0.75) node[
                              label = {[xshift=.0cm, yshift=.25cm]below:\footnotesize$ q_2 =0.97$}] {};
                              
    \draw [dashed, very thick, blue, -{Stealth[scale=1.2]}] (j-1) -- (0,0 |- j-1) node[
                              label = {[xshift=-1.185cm, yshift=-0.4cm]above:\footnotesize${\eta_2} \approx 3.476$}] {};

    \endscope
  \endscope

%% file: 06_Discussion.tex
\clearpage
\newpage

\section{Discussion of Related Work}
\label{sec:diversity_zero_one}

We provide a discussion of related literature, 
and outline some possible future work.
%

%

%
\subsection{A diversity of diversities}
There are {\em many} proposed measures of ensemble diversity.  \citet[Section 8.5]{kuncheva2014} 
categorises measures along three lines: (1) diversity as a characteristic of the classifiers, not considering the combiner or the true label; (2) diversity as a characteristic of the classifiers and the combiner, but {\em not} using the true label; and (3) diversity as a characteristic of the classifiers, the combiner, {\em and} the true label.  Our framework clearly supports the second idea, for losses where Definition~\ref{def:gen_bv} holds (e.g., Bregman divergences), and the third, for losses where it does not (e.g., 0/1 loss, absolute loss).


\subsection{Good and Bad Diversity in Majority Voting Ensembles}\label{subsec:propertiesZeroOne}


With the 0/1 loss, \autoref{the:gen_bvd} does not apply. As a consequence, we have to use the {\em effect} decomposition, in 
\autoref{the:diversity_effect}.  
As a reminder, we still consider the {\em measurement} of diversity in the same form as other losses: $\ED[\averagei \Lzeroone(\qbar,q_i)]$. However, the {\em effect} that this quantity has on the ensemble risk is given by the {\em diversity-effect} term.
This can be related to the idea of `good' and `bad' diversity \citep{brown2010good}.\\

\noindent {\bf ``Good'' vs ``Bad'' diversity.}
%
\citet{brown2010good} showed that, restricting the label to $y \in \{-1,+1\}$, and $\qbar$ as a majority vote, the following holds:
\begin{align}
    \underbrace{\mystrut{1.4em}\Lzeroone(y, \qbar)}_{\textnormal{ensemble~loss}} ~=~ \underbrace{\averagei \Lzeroone(y, q_i)}_{\textnormal{average~loss}}  - \underbrace{{y \qbar \averagei \Lzeroone(\qbar, q_i)}}_{\textnormal{good/bad~diversity}}.
    \label{eq:goodbadambiguitydecomp}
\end{align}
The similarity to the ambiguity decompositions from earlier is self-evident. The difference is the $y\qbar$ preceding the second term. 
%
{\bf The sign of $y\qbar$ mediates the effect of the diversity.}
If $y\qbar=+1$ (equivalent to saying $\qbar=y$), the term is non-negative and therefore acts to subtract from the average error. When this occurs, it is referred to as ``good'' diversity.
When $y\qbar=-1$, the diversity adds to the error, referred to as ``bad'' diversity.
Comparing this to the Ambiguity-effect decomposition (\autoref{prop:ambig_effect}) we can relate the two as follows.

\begin{corollary}
For 2-class problems, $y\in \{-1,+1\}$, the ambiguity-effect is Brown \& Kuncheva's good/bad diversity term, averaged over the data distribution.
\begin{align}
    \Exy\Big[\underbrace{{Y \qbar \averagei \Lzeroone(\qbar, q_i)}}_{\textnormal{good/bad diversity}}\Big] = \underbrace{\averagei\Big[ R_{0/1}(q_i) - R_{0/1}(\qbar) \Big]}_{\textnormal{ambiguity-effect}} .
\end{align}
\end{corollary}

The concepts of Good/Bad diversity were generalised to the multi-class case by \citet{didaci2013diversity}, and the same relation applies.
%
The idea was used in~\cite{bian2019does} to define diversity in weighted ensembles,
giving 
the same form as \cite{didaci2013diversity}. \\

{\noindent\bf When the effect of diversity is negative.}
The definition of diversity-effect is simple to state: {\em the expected difference between the average individual risk and the ensemble risk.}  If the ensemble risk performs worse than the individuals, this will be negative, and hence the effect of diversity will be to {\em increase} the 0/1 risk.  
Assume a majority voting ensemble of classifiers, making independent errors on a $k$-class  problem,
with each model predicting the correct class with probability $p$. 
If $p > 0.5$, then the ensemble 0/1 risk is guaranteed to be less than the individual 0/1 risk
\citep{lam1997application}.
In our context
this means the diversity-effect will be non-negative, and thus the effect of diversity is to reduce ensemble risk.
However, this idealised scenario will certainly not always be possible.  \citet[Section 8.3]{kuncheva2014} details carefully constructed scenarios where more diversity is
associated with {\em increasing} ensemble risk, i.e., the bad diversity outweighs the good diversity.


\subsection{The Correlation of Diversity and Classification Accuracy}

Figure~\ref{fig:diversity_toy_plot} is a toy ``accuracy/diversity'' scatter plot, in the style popularised by \citet{kuncheva2003measures}; it shows the correlation of a diversity measure to  accuracy (0/1 loss). 
The x-axis is some diversity measure, and the y-axis is 
$\averagei R_{0/1}(q_i) - R_{0/1}(\qbar)$,
i.e., the difference between the average individual error and the ensemble error. The higher this value, the more the ensemble outperforms the average individual model.
A higher correlation to diversity is seen to be a more successful diversity measure, as it explains the performance improvement. 
Figure~\ref{fig:mlp_scatter} shows two such plots for {\em real} data, for the cross-entropy diversity.  We estimate diversity on validation data, and the 0/1 loss on a final test set.
The caveat here, is that we have assumed access to class probability estimates, for computing the cross-entropy diversity term. This may not always be possible.

\begin{figure}[h]
    \centering
    \includegraphics[width=0.45\textwidth]{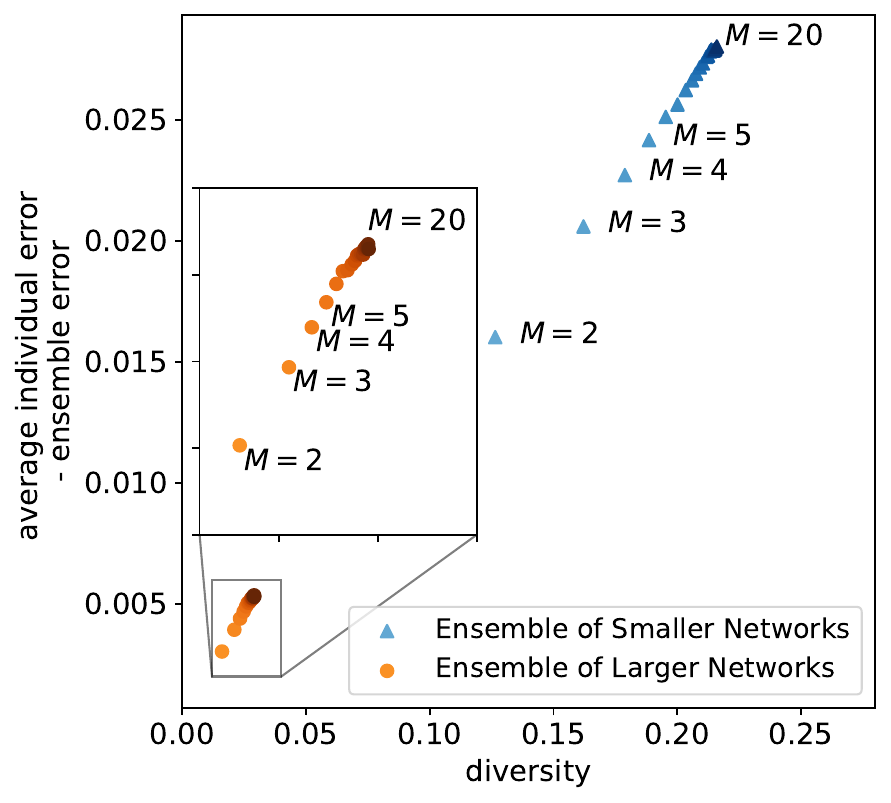}
    ~~~~~~~\includegraphics[width=0.46\textwidth]{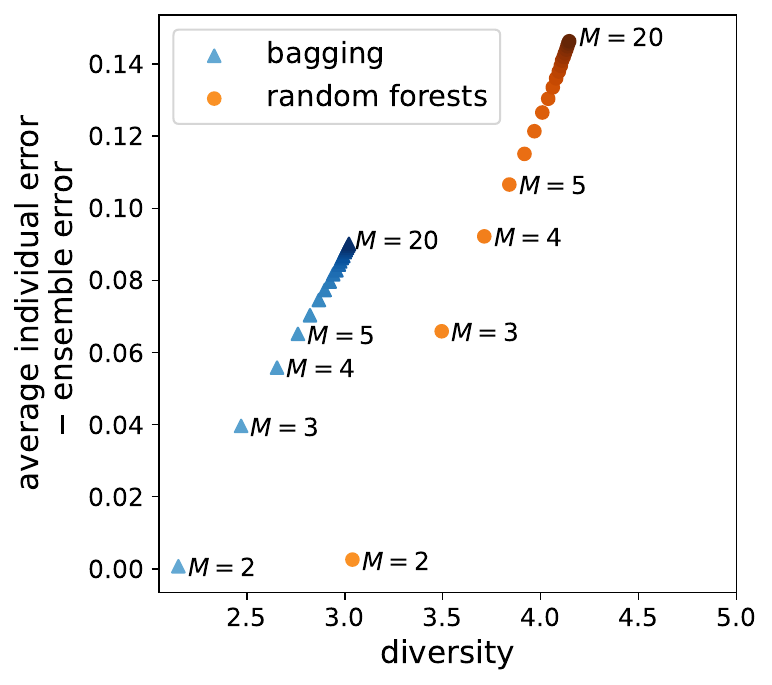}
\caption{Accuracy/Diversity plots. LEFT: Bagged MLPs (corresponds to experiment in Figure~\ref{fig:decompose_MLPcrossentropy}). RIGHT: Decision tree ensembles, Bagging vs Random Forests.}
\label{fig:mlp_scatter}
\end{figure}

We see a strong correlation for both ensembles.
The smaller networks (blue triangles, Pearson's $r^2=0.998$)
have greater diversity than 
the larger networks (orange circles, $r^2=0.992$).
The plot must be read in the context that overall, the larger networks significantly outperformed the smaller networks---it simply shows that the performance came from more powerful base models, as opposed to their diversity. 
%
%

In the right hand figure, we compare Bagging decision trees (unlimited depth) and Random Forests on MNIST.  In both cases the trees are combined by obtaining probabilities and combining via normalised geometric mean.  
Again we see strong correlation of diversity and performance gain.
Bagging has a correlation $r^2\approx 0.996$, whilst Random Forests has $r^2\approx 0.999$.
For a fixed $m$, we can compare corresponding points, where the {\em only difference is the additional split-point randomisation of the forest.}
At $m=20$, RF provides a reduction in generalisation error of $\approx 14.5\%$, versus only $\approx 9\%$  for Bagging.
It interesting to note this is solely due to increased diversity generated by random feature splits.

One might wonder {\em why} there is such a strong correlation in both cases. If we remember the alternative view of the same experiment, Figure~\ref{fig:decompose_MLPcrossentropy}, we see that bias/variance are constant, and it is {\em only the diversity that changes}.  When any change is observed in the overall ensemble cross-entropy, we know it is caused by a change in diversity. Therefore, if we can assume strong correlation between the  ensemble cross-entropy and the 0/1 loss, then there will be a similar strong correlation between diversity and 0/1 loss.

We might now wonder, with this diversity measure, {\em will we always see a strong correlation between diversity and reduction in 0/1 loss?}  
The answer is no, for a very good reason that highlights a critical point in our understanding of diversity.
In Figure~\ref{fig:something} we fix at $m=10$ bagged trees, and vary their depth.  The expected loss reduces---however, now it is not solely due to diversity. Now, the \bias~and \variance~also change rapidly, and the correlation of 0/1 loss/diversity is much lower.

\begin{figure}[h]
    \centering
    \includegraphics[width=\textwidth]{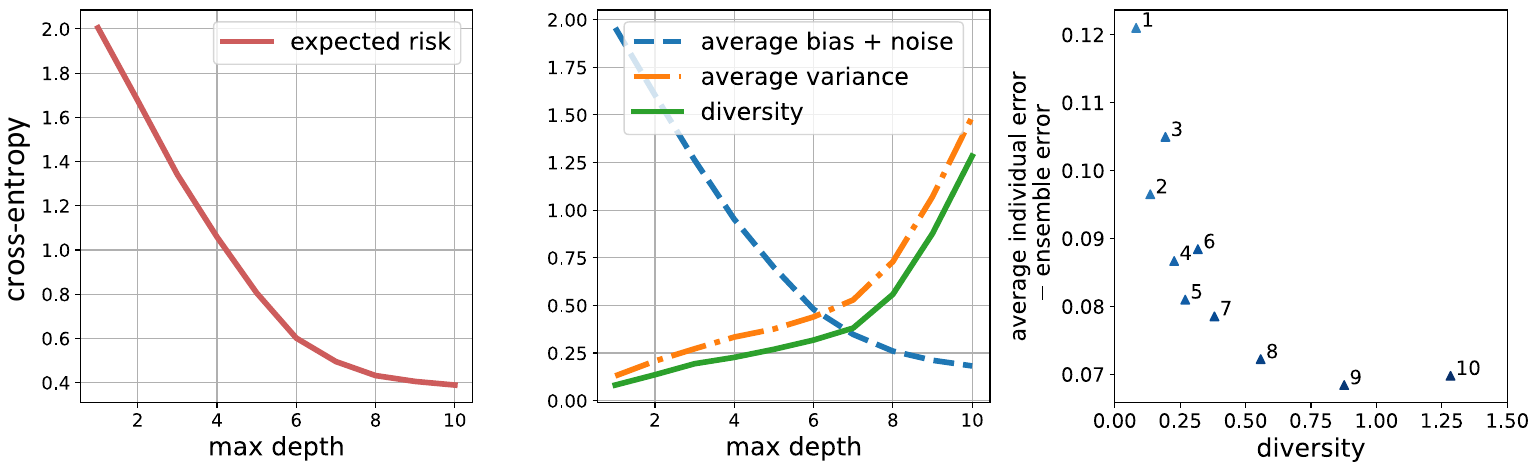}
    \caption{Bagging $m=10$ trees, varying depth, and correlation is now $r^2\approx 0.59$.} 
    \label{fig:something}
\end{figure}

When varying any other parameter than $m$, one {\em should not expect} to see a strong correlation of performance improvement and diversity. This is because, if we vary any parameter that alters individual capacities, then the average bias/variance also changes, and diversity is not the only factor in play.  The overall performance is decided by a 3-way trade-off, just as there is a 2-way trade-off of bias/variance in single models. It would be interesting to explore this with ensembles of very deep neural networks, where the bias/variance trade-off seems to not act as classical theory predicts \citep{belkin2018reconciling}.


\newpage

\subsection{Future work}

{\noindent\bf Re-visiting previous work involving heuristic diversity measures.}
In the introduction, we outlined several areas that had made use of heuristic diversity measures. Examples were given for diversity helping defend adversarial attacks, or counteracting covariate shift, or in forming more computationally efficient models.
%
%
Any one of these might be repeated, using diversity derived from the loss function, as in our framework.\\

{\noindent\bf Forcing diversity.} A natural line of research might be to {\em enforce diversity} in some sense, i.e., using our diversity measures as regularisers during the construction of an ensemble.  Negative Correlation (NC) Learning \citep{liu1999ensemble} encourages diversity in regression ensembles.  \citet{brown2006managing} showed this is in fact exploiting the squared loss ambiguity decomposition, Equation~\eqref{eq:ambiguity}.   The Bregman ambiguity decomposition, Equation~\eqref{eq:bregman_ambiguity}, implies that the NC algorithm is a special case of a wider family of diversity-encouraging losses---the case for cross-entropy was explored in \citet{webb2021ensemble}.\\ 

{\noindent\bf Diversity for Margin losses.} 
Margin losses (e.g., logistic or exponential loss) are an important class of losses, best known in the context of {\em Gradient Boosting} algorithms.
\cite{wood2022bias} analysed bias-variance decompositions for such losses. 
Using these results, it would be possible to obtain bias-variance-diversity decompositions for margin losses, and thus potentially gain insight into boosting algorithms.
%
However, \citet{mease2008evidence} showed strong evidence that the additive model form in AdaBoost/LogitBoost results in a disconnect between the surrogate margin loss and our true objective, the 0/1 loss.  In particular, the surrogate loss can go {\em up} (sometimes exponentially fast) whilst the 0/1 loss on a hold-out sample is going {\em down}.   
This implies that any analysis of the surrogate (including loss decompositions) does not necessarily give meaningful insights for the 0/1 loss.
Furthermore, with boosting, the individual models are more naturally interpreted as learning to correct the errors of previous ensemble members rather than perform well in their own right, making interpretation of the average bias term problematic.  Extending our framework for these cases could be an interesting new interpretation of this popular class of algorithms.\\

{\noindent\bf Diversity in Mixtures of Experts.} As it stands, our framework does not apply for a mixture of experts model, where the models are gated by an additional input-dependent model, itself learnt from data.  Future work might consider diversity for this scenario---perhaps measuring diversity of predictions in small regions of the space, where the gated weights of members are reasonably constant.\\

{\noindent\bf Understanding or Encoding Probabilistic Assumptions.} The concept of diversity is one way of measuring dependencies between predictions.  An explicit {\em probabilistic} assumption of conditional independence (sometimes known as the {\em Dawid-Skene} model) would seemingly imply some degree of diversity.  It is not clear how the loss function view of diversity (which we present) relates to this probabilistic view.  Given the duality between losses/log-likelihood, and that Bregman divergences can be seen as KL-divergences of probability densities, this should be possible and may yield interesting insights.

%% file: 07_Conclusions.tex
\newpage

\section{Conclusion}

We have presented a unified theory of ensemble diversity.  
A key insight is that it is not the task (e.g., classification/regression) that matters, but the {\em loss function.}
We demonstrated that
diversity can be seen as a {\em hidden dimension in the bias-variance decomposition of an ensemble loss.}
Diversity emerges naturally with this point of view---the exact functional form is specific to the loss being used, but the decompositions have a common structure applicable for a wide range of losses:
\begin{equation}
        \textbf{expected~loss} ~=~ (\textbf{average bias)} ~+~ (\textbf{average variance}) ~-~ (\textbf{diversity}).\notag
\end{equation}
This gives a clear relationship between the ensemble performance and its diversity. 
The only other scenario where this was previously available is for squared loss with an arithmetic mean combiner \citep{ueda1996generalization}. Our framework is an alternative in this case, but  generalises the notion of diversity to a far wider range of losses, including the cross-entropy, and Poisson regression losses, but more generally any Bregman divergence.
The decomposition requires the use of a particular combination rule, specific to the loss at hand, which we call the {\em centroid combiner.} This combiner turns out to correspond to several well-known combiners already in the literature---e.g., for cross-entropy, it is the normalized geometric mean.  This generalises the idea of ensemble ``averaging'' to many other scenarios, explained as averaging in a dual coordinate system defined by a Bregman divergence.

For losses where an additive bias-variance decomposition does not exist, we adopted the approach of \citet{james1997generalizations}, to instead measure the {\em effects} of bias, variance, and diversity.
which turn out to be dependent on the label distribution.  The case of 0/1 loss is particularly interesting---we show that, not just for majority voting, but for any combiner, the effect of diversity is necessarily {\em a label-dependent} quantity.

We therefore have a broad and precise formulation of diversity for a wide range of supervised learning scenarios. 
This challenge has been referred to as the ``holy grail'' of ensemble learning \citep[pg 100]{zhou2012ensemble}, an open question for over 30 years.
The answer we provide phrases diversity as a measure of {\em model fit}, in precisely the same sense as bias/variance.
%
%
Thus, we should not be aiming to ``maximise diversity'' as so many works aim to do.
Instead, just as bias and variance change with model characteristics, the same applies to diversity, and we have to manage the three-way {\em bias/variance/diversity trade-off}.\\

\noindent Code for all experiments at: \url{https://github.com/EchoStatements/Decompose}


\acks{Funding in direct support of this work: EPSRC EP/N035127/1 (LAMBDA project) and EP/T026995/1 (EnnCore project). Mikel Luj\'an holds an Arm/RAEng Research Chair Award and a Royal Society Wolfson Fellowship.\\~\\GB would like to thank LK and FR for a career’s worth of inspiration \& support.}

\appendix

%% file: 99_A_AdditionalExperiments.tex
\newpage

\section{Additional Experimental details}

\subsection{Additional results}
\label{app:additionalexperiments}

We present additional results on several datasets. 
We emphasize that we make no claims on empirical superiority of any one method over any other. We simply use these toy datasets as illustrative examples of how the risk components can be estimated.\\

{\noindent\bf Squared loss:} We use California housing data---Table~\ref{tab:ensembleCalifornia} shows results from three ensembles (each $M=30$ regression trees), compared to a single  tree.
We use Bagging with constrained trees (max depth 8) and compare against unlimited depth trees, and a Random Forest.

\begin{table}[ht]
\centering
\begin{tabular}{@{}llll@{}}
    \toprule
    \shortstack[l]{\bf \small Single tree \\\bf \small(depth 8)} & \shortstack[l]{\bf\small Bagging\\ \bf\small (depth 8)} & \shortstack[l]{\bf\small Bagging\\ \bf\small (unconstrained)} & \shortstack[l]{\bf\small Random Forest\\\small\phantom{x}} \\ \midrule
    $0.47$ & $0.35$ & $0.30$ & $0.28$ \\
        \bottomrule
\end{tabular}
\caption{California housing data: MSE of a single tree versus ensembles of 30 trees.}
\label{tab:ensembleCalifornia}
\end{table}

We observe that the Random Forest is the best choice here, followed up closely by the unconstrained Bagging.
Figure \ref{fig:bvdplot_regression} explains their performance
by decomposing risk into bias, variance, and diversity---also showing how the components change as we grow the ensemble.

\begin{figure}[h]
    \centering
    \includegraphics[width=.9\textwidth]{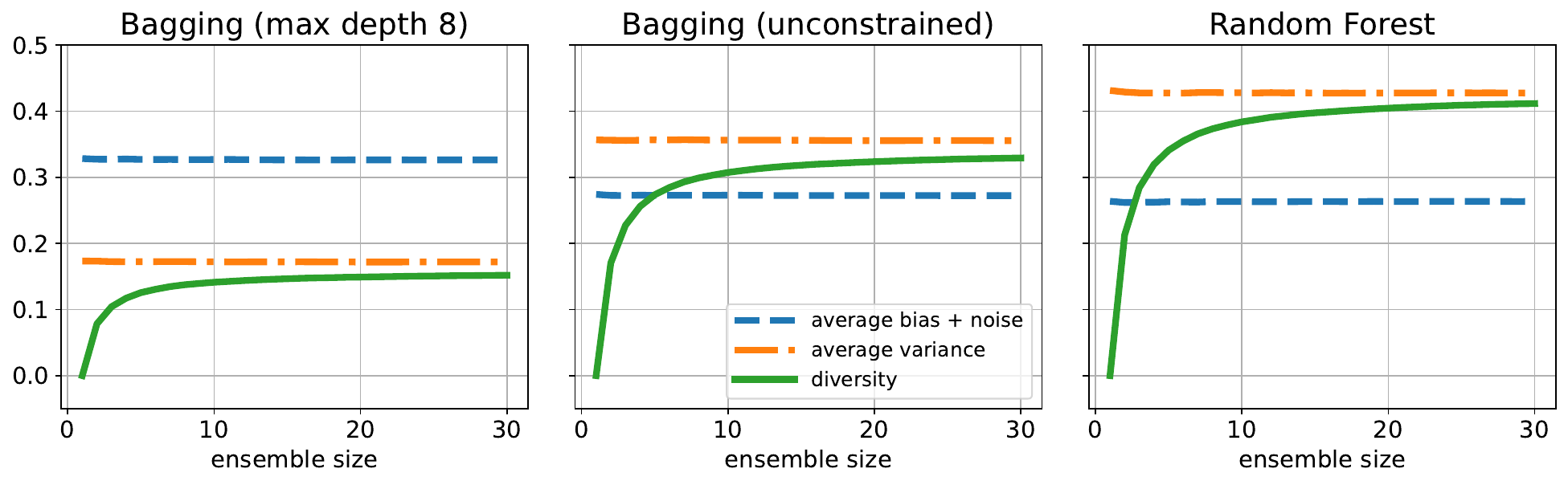}
    \caption{Decomposing the expected risk of three decision tree ensembles.}
    \label{fig:bvdplot_regression}
\end{figure}

We observe the same behaviour as in Figure \ref{fig:californiaBVDexample}.  The diversity increases with $m$, and is upper-bounded by the value of the average variance.  A higher average variance effectively raises the ``ceiling'' to which diversity can rise.  
The average variance is higher as we move from depth-limited trees to unlimited depth, and higher again with the random split-points in the Random Forest (here we use the square root of the number of features).  The higher average variance is compensated for by the diversity, causing Random Forest to be the best option. It is notable that for large ensembles, the expected risk of the ensemble is almost entirely due to the value of the average bias ($\approx 0.28$ in the case of unconstrained trees), with diversity having essentially cancelled out the average variance of the individual models.
This behaviour is not just a quirk of this data set, in fact it holds as long as the individuals are all from the same model family, i.e., the ensemble is {\em homogeneous}---the general case is discussed in Section~\ref{sec:furtherproperties}.\\


{\noindent \bf Cross-entropy:} In Figure~\ref{fig:lotsofdata}, we have additional accuracy/diversity plots for neural network ensembles on classification problems ($n$ examples, $d$ features, $k$ outputs/classes).
Datasets: Phoneme ($n=5404$, $d=4$, $k=2$), Landsat ($n=6435$, $d=36$, $k=6$), Spambase ($n=4601$, $d=57$, $k=2$), South German Credit (2019 version: $n=1000$, $d=20$, $k=2$). 
 
In each case, the squared Pearson's correlation coefficient is shown in the legend.
The following configuration was used in all MLP experiments:

\begin{itemize}    \setlength{\itemsep}{-3pt}%
    \item learning rate: 0.1 (Stochastic gradient descent)
    \item num epochs: 50 (MNIST), 200 (other data sets)
    \item hidden layer size (20 small/100 larger)
    \item number of trials: 100
\end{itemize}
where each trial uses a 90\% sub-sample of the full training data, as outlined in Figure~\ref{fig:protocol}.

\begin{figure}[ht]
    \centering
    \includegraphics[width=\textwidth]{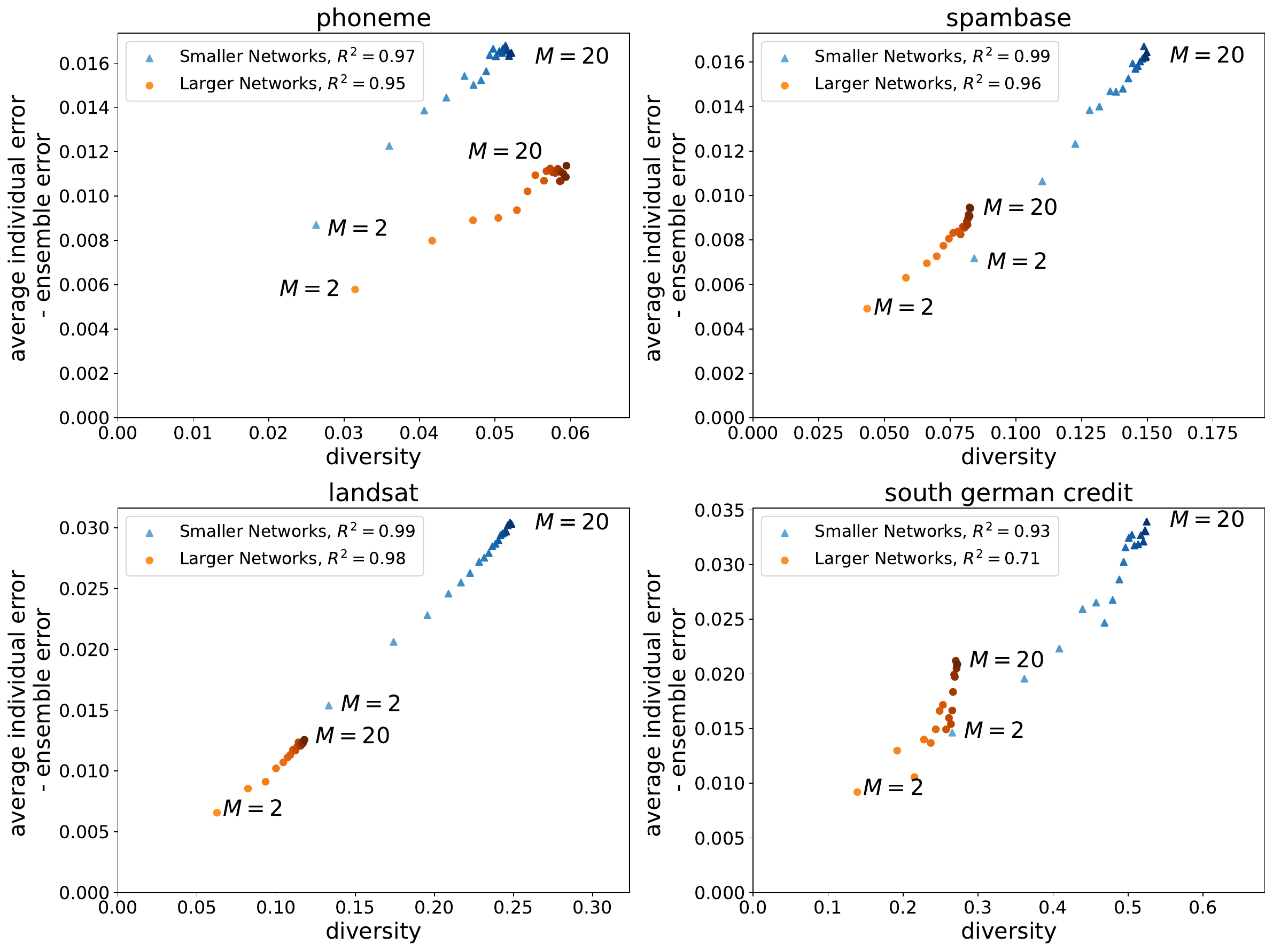}
    \caption{Error/diversity plots across four data sets, comparing ensembles of small MLPs (20 hidden nodes, blue dots) versus large (100 hidden nodes, orange dots).}
\label{fig:lotsofdata}
\end{figure}

\newpage
\subsection{Methodology for Estimating Bias, Variance, and Diversity}\label{app:estimating}

We present our methodology for estimating the bias, variance and diversity from data.
Algorithm~\ref{alg:data_collection} shows the procedure for experiments where we estimate diversity of ensembles of different sizes, such as in the experiments for Bagging and Random Forests. 
Notably, an ensemble of size $m+1$ is created by using the members of the ensemble of size $m$, rather creating a new ensemble of size $m+1$ from scratch. We also present a visualisation of the sub-sampling scheme used for Bagging in Figure~\ref{fig:protocol}.

\SetKwComment{Comment}{/* }{ */}

\begin{algorithm}[ht]\label{alg:data_collection}
$m$ : number of ensemble members\\
$n$ : number of test points\\
$t$ : number of trials\\
$D_k$ : the training set to use in trial $k$\\
~\\
{\bf Arguments}: $model$, $train\_data$, $test\_data$ \\ 
{\bf Output}: $test\_preds$: an array of model predictions of size $t\times m\times n$.\\ 
    \For{$k \in \{1, \ldots, t\}$}{
        \For{$j \in \{1, \ldots, n\}$}{
        $D_k \gets$ 90\% of $train\_data$, sampled without replacement\;
            \For{$i \in \{1, \ldots, $m$\} $ }{
                $member\_data$ $\gets$ bootstrap from $D_k$\;
                $i$th ensemble member $\gets$ copy of $model$ trained on $member\_data$\;
                $test\_preds$[k, i, j] $\gets$ prediction of $test\_data$ from the $i$th ensemble member, in the $k$th trial, $j$th test data point\;}}}
    \caption{Algorithm for collecting data, later used to estimate diversity of a Bagging ensemble, whilst varying ensemble size $m$.}
\end{algorithm}

\begin{figure}[ht]
    \centering
    \includegraphics[width=10.5cm]{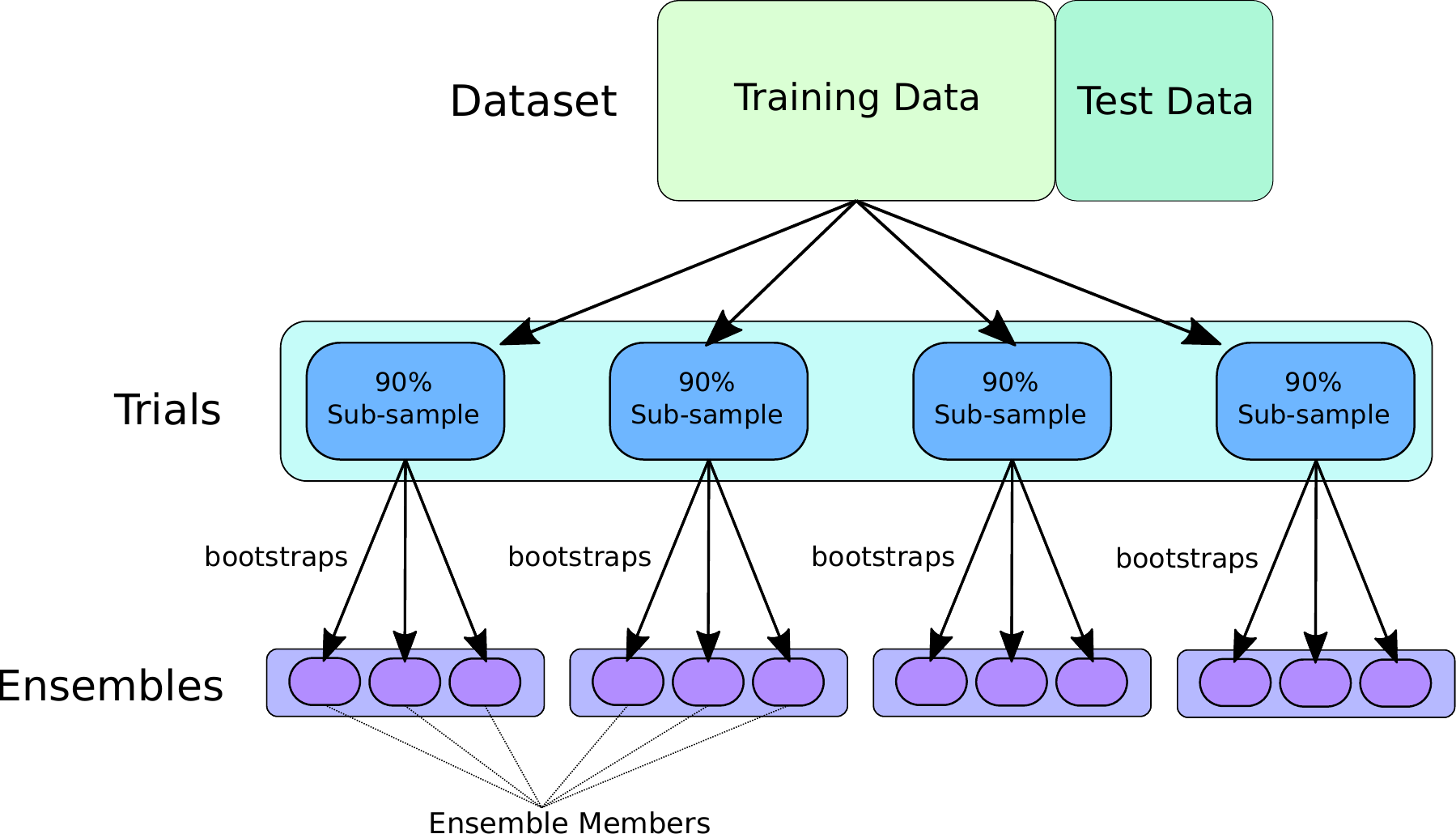}
    \caption{Visualisation of the sub-sampling scheme used for Bagging ensembles.}
    \label{fig:protocol}
\end{figure}

The result of Algorithm~\ref{alg:data_collection} is an array of size $(t, m, n)$. 
For Bregman divergences, the average bias and average variance are approximated by first estimating the left Bregman centroid for each ensemble member, replacing the expectation over the random variable $D$ with an average over $t$ trials. We use notation ${D}_k$ as the full training data in the $k$th trial, from which the Bagging algorithm can be applied.  The centroid for model $i$ on a single test point $\vectorx$ is therefore estimated as,
\begin{align*}
   \centroid{q}_i^{est}(\vectorx) =   [\nabla \phi]^{-1} \left( \frac{1}{t} \sum_{k=1}^t \nabla \phi( q_i(\vectorx; {D}_k)) \right) \approx \centroid{q}_i(\vectorx)
\end{align*}
With this estimate of $\centroid{q}_i(\vectorx)$, the average bias and average variance are computed as
\begin{equation}
    \textnormal{average bias} \approx \frac{1}{m} \frac{1}{n}  \sum_{i=1}^m \sum_{j=1}^{n} \BregmanGen{\phi}{y_j}{ \,\centroid{q}_i^{est}(\vectorx_j)},
\end{equation}\break
\begin{equation}
\textnormal{average variance} \approx \frac{1}{t} \frac{1}{m} \frac{1}{n}  \sum_{k=1}^t \sum_{i=1}^m \sum_{j=1}^{n} \BregmanGen{\phi}{ \centroid{q}_i^{est}(\vectorx_j)}{\,q_i(\vectorx_j; {D}_k)}.
\end{equation}
Diversity is estimated similarly, with $\qbar$ defined as the centroid combiner rule applied to the $m$ ensemble members in a given trial:
\begin{equation}
\textnormal{diversity} \approx \frac{1}{t} \frac{1}{m} \frac{1}{n}  \sum_{k=1}^t \sum_{i=1}^m \sum_{j=1}^{n} \BregmanGen{\phi}{ \qbar(\vectorx_j, {D}_k)}{\,q_i(\vectorx_j; {D}_k)}.
\end{equation}

%% file: 99_B_Exposition.tex
\newpage

\section{Further explanations}

We present further exposition on certain points within the paper. The material in this appendix is not novel, but expands on selected aspects to support our argument.

\subsection{Is it ``squared bias'', or just ``bias''?}
\label{app:squaredbias}

\citet{geman1992neural} presented a squared loss bias-variance decomposition, Equation~\eqref{bvsqloss}. One of the terms is $(\EDb{q}-\ystar)^2$. They, and many subsequent papers, refer to this as ``squared bias'', or ``bias squared''.  {\em This is a misnomer}, as we will now argue.

The root of the issue is that Geman et al. borrowed terminology from classical statistics. 
For some population value $\theta$, and an estimator $\hat{\theta}$, classical statistics refers to the quantity $\mathbb{E}[\hat{\theta}]-\theta$ as the {\em bias} of $\hat{\theta}$. Thus, by analogy,
they refer to $(\EDb{q}-\ystar)^2$ as the {\em squared} bias of estimator $q$, where $\ystar$ is assumed to be the population value.

The reason this is a misnomer---is that similar decompositions are known for other losses, and the square is {\em not present}, e.g. \cite{heskes1998bias, pfau2013, wood2022bias}. 
Relevant to our paper, the ensemble bias-variance-covariance decomposition of \citet{ueda1996generalization} unfortunately also inherited the ``squared'' misnomer, as we explained  in sub-section~\ref{subsubsec:squaredbvc}.
These results were proven long after the publication of \citet{geman1992neural}, thus we cannot blame the authors, especially since the term ``bias'' is overloaded in science (e.g. inductive, sampling, selection, confirmation).  {\em  So, in short, the term ``squared bias'' is a misnomer, and throughout our paper we refer to just ``bias''.}

\subsection{The Ambiguity Decomposition is equivalent to Bias-Variance}
\label{app:specialcase}

In the bias-variance decomposition, the random variable $D$ is commonly assumed to be over all possible training sets of a fixed size $n$. This is a convention introduced by \citet{geman1992neural}.  However this is not necessarily the case, and $D$ can be over any stochastic quantity involved in construction of the model (e.g. random initial weights for neural networks), or indeed over a finite set of {\em pre-constructed models}.  In this latter case, we see an equivalence between the bias-variance decomposition and the ambiguity decomposition---in the sense that if one exists, then the other must also.
This can be seen most easily by considering a simpler form of the bias-variance decomposition, with no noise: 
\begin{equation}
    \EDb{ \ell(y,q) } =
       \ell(y,\centroid{q})
        + \EDb{\ell(\centroid{q},q)}.
        \label{eq:simple_bv}
\end{equation}

The relation still holds with noise, but is easier to explain with this assumption. Now, define $D$ as a discrete random variable with an event space consisting of $m$ different pre-constructed models, $\{q_i\}_{i=1}^m$, and the probability assigned to each event is uniform, i.e., $1/m$ for each event $q_i$. This means the expectation $\ED$ becomes an average $\averagei$.  The centroid is still defined in the same way, but we use notation $\qbar$ to denote the fact that it concerns a finite set of predictions: $\qbar \defeq \argmin_{z\in\bregmanDomain} \averagei\ell(z,q_i)$.  This means the expression above becomes:

\begin{equation}
    \averagei{ \ell(y,q_i) } =
       \ell(y,\qbar)
        + \averagei{\ell(\qbar,q)}.
\end{equation}
A simple rearrangement of terms yields the ambiguity decomposition.

\newpage

\subsection{Properties of Heskes' Decomposition for the Cross-Entropy}\label{app:decomposiingCrossEntropy}

\citet{heskes1998bias} presents a bias-variance decomposition for the KL-divergence between probability densities, which he extends to expose a noise term, giving a noise-bias-variance decomposition for the cross-entropy.  For Multinoulli distributions, we claim this is equivalent to that we present in Equation~\eqref{bv:KL}, 
which is important as it then matches our unified bias-variance form. 
For clarity, we restate our KL decomposition here:
\begin{equation}
    \underbrace{\mystrut{12pt} \ED\Big[ \EXY\left[ K\left(\vectorY\mid\mid\vectorq\right) \right] \Big] }_{\substack{\textnormal{expected}~\textnormal{risk}}} 
        = 
        \Ex\Bigg[ \underbrace{\mystrut{12pt} \mathbb{E}_{\vectorY\mid \vectorX}\left[ K\Big(\vectorY\mid\mid\vectorYstar\Big) \right] }_{\textnormal{noise}} 
        +  \underbrace{\mystrut{12pt}K\Big( \vectorYstar\mid\mid \centroid{\vectorq} \Big)}_{\textnormal{bias}} 
        + \underbrace{\mystrut{12pt} \ED \left[ K\Big( \centroid{\vectorq} \mid\mid \vectorq\Big) \right] }_{\textnormal{variance}} \Bigg],  \label{app:bv:KL}
\end{equation}
where ${\bf Y}$ is random variable, over one-hot encoded indicators of the true class for an input point $\vectorx$.
If we take Heskes' expression  \citep[Eq 6]{heskes1998bias} and assume Multinoulli distributions, we get a decomposition of the expected cross-entropy at a single point $\vectorx$, between a true class distribution $\vectorystar$ and a model $\vectorq$.  In our notation this reads:
\begin{equation}
       -\mathbb{E}_D\Big[ \vectorystar\cdot\ln {\vectorq} \Big] = \underbrace{\mystrut{12pt} -\vectorystar\cdot\ln\vectorystar}_{\textnormal{noise}}  + \underbrace{\mystrut{12pt} \KL{\vectorystar}{\centroid{\vectorq}}}_{\textnormal{bias}} + \underbrace{\mystrut{12pt} \ED\Big[{ \KL{\centroid{\vectorq}}{\vectorq}}\Big]}_{\textnormal{variance}},
\end{equation}
where $\centroid{\vectorq} \propto \textnormal{exp}(\ED[\ln\vectorq])$ is the normalized geometric mean of the model distribution.

To show the equivalence, we first note that \eqref{app:bv:KL} has an expectation over $P({\bf X})$. We can restate this at a single  ${\bf x}$,
\begin{equation}
    \underbrace{\mystrut{12pt} \ED\Big[ \mathbb{E}_{\vectorY|{\bf X}={\bf x}}\left[ K\left(\vectorY\mid\mid\vectorq\right) \right] \Big] }_{\substack{\textnormal{expected}~\textnormal{loss}}} 
        = 
        \underbrace{\mystrut{12pt} \mathbb{E}_{\vectorY\mid \vectorX=\vectorx}\left[ K\Big(\vectorY\mid\mid\vectorystar\Big) \right] }_{\textnormal{noise}} 
        +  \underbrace{\mystrut{12pt}K\Big( \vectorystar\mid\mid \centroid{\vectorq} \Big)}_{\textnormal{bias}} 
        + \underbrace{\mystrut{12pt} \ED \left[ K\Big( \centroid{\vectorq} \mid\mid \vectorq\Big) \right] }_{\textnormal{variance}} ,
\end{equation}
where $\vectorystar=\mathbb{E}_{\vectorY|{\bf X}=\vectorx}[\vectorY]$ is the true class distribution at $\bf x$.
The second and third terms on the right are equal to those in Heskes' decomposition. For the noise term we note,
\begin{eqnarray}
    \mathbb{E}_{\vectorY\mid \vectorX=\vectorx}\left[ K\Big(\vectorY\mid\mid\vectorystar\Big) \right] &=& \mathbb{E}_{\vectorY\mid \vectorX=\vectorx}\left[ \vectorY\cdot\ln\frac{\vectorY}{\vectorystar} \right] \\
    &=& \mathbb{E}_{\vectorY\mid \vectorX=\vectorx}\left[ \vectorY\cdot\ln{\vectorY} \right] -  \mathbb{E}_{\vectorY\mid \vectorX=\vectorx}\left[ \vectorY\cdot\ln{\vectorystar} \right] \\
    &=&  -  \mathbb{E}_{\vectorY\mid \vectorX=\vectorx}\left[ \vectorY\right] \cdot\ln{\vectorystar} \\
    &=& -\vectorystar\cdot\ln\vectorystar,
\end{eqnarray}
where we used the convention $0\ln 0=0$.  Since the right hand sides are equal, therefore the left hand sides are equivalent too.

\newpage

\subsection{The Importance of Parameter Encoding in the KL-divergence}\label{app:parameterEncoding}

Consider the KL-divergence between two discrete probability vectors of length $k$,
\begin{equation}
    \KL{\vectorp}{\vectorq} = \sum_{c=1}^k \vectorp^{(c)}\ln \frac{\vectorp^{(c)}}{\vectorq^{(c)}}
\end{equation}
There are two ways in which this can be expressed as a Bregman divergence, either using the full-length probability vectors, $\vectorp \in \mathbb{R}^k$ and using the generator 
\begin{align}
    \phi_{\textnormal{full}}(\vectorp) = \sum_{c=1}^k \vectorp^{(c)} \ln \vectorp^{(c)}\label{eq:kl_full}
\end{align}
or using the minimally parameterised vectors, $\vectorptilde \in \mathbb{R}^{k-1}$, where the last entry is omitted: 
\begin{align}
    \phi_{\textnormal{min}}(\vectorptilde) = \sum_{c=1}^{k-1} \vectorptilde^{(c)} \ln \vectorptilde^{(c)} + (1 - \sum_{c'=1}^{k-1} \vectorptilde^{(c')} ) \ln (1 - \sum_{c'=1}^{k-1} \vectorptilde^{(c')} ). \label{eq:kl_min}
\end{align}

Given two probability vectors in the appropriate form, either formulation gives a Bregman divergence is equivalent to the KL-divergence between the vectors of class probability estimates~\citep{nielsen2009sided}, i.e.,
        \begin{align*}
        \BregmanGen{\phi_{\textnormal{full}}}{\vectorp}{\vectorq} = \BregmanGen{\phi_{\textnormal{min}}}{\vectorptilde}{\vectorqtilde} =  \KL{\vectorp}{\vectorq}.
\end{align*}
We use the minimal parameterisation, as it exhibits a desirable property. In particular, it is necessary to use the second to ensure that the Bregman centroid is always a valid distribution on the probability simplex. In this case, the centroid combiner is the normalised geometric mean, as we now demonstrate.\\

{\noindent \bf Minimally parameterized vectors:}
To show this we consider $m$ minimally parameterised vectors, ${\bf q}_i \in \mathbb{R}^{k-1}$ (note that we have dropped tilde above $\vectorq$ for simplicity).
Our claim is that the centroid combiner is the normalised geometric mean
$\vectorqbar = [\nabla\phi]^{-1}\left(\frac{1}{m}\sum_{i=1}^m \nabla\phi ({\bf q}_{i}) \right)$, 
is of the form
\begin{align}
    {\vectorqbar}^{(c)} 
        = [\nabla\phi_{\textnormal{min}}]^{-1}\left(\frac{1}{m}\sum_{i=1}^m \nabla\phi_{\textnormal{min}} ({\bf q}_{i}) \right)
    = \frac{\prod_{i=1}^m {\vectorq_i^{(c)}}^\frac{1}{m}}{ \sum_{c'=1}^k \prod_{i=1}^m {\widehat{\vectorq}^{(c')\frac{1}{m}}} },\label{eq:geo_mean}
\end{align}
where $\widehat{\vectorq}$ denotes the extension of the $k-1$ length vector into a full $k$ length probability vector.
Plugging in the gradients from Table~\ref{tab:bregmancentroids}, we start with
\begin{align*}
     {\vectorqbar}^{(c)} 
    &= \frac{\exp \left(\averagei \ln 
            \frac{\vectorq_i^{(c)}}
            {1-\sum_{c'=1}^{k-1} \vectorq_i^{(c')}}\right) }
         { 1+ \sum_{c'=1}^{k-1} \exp \left( \averagei \ln \frac{\vectorq_i^{(c')}}{1-\sum_{c''=1}^{k-1} \vectorq_i^{(c'')}}\right)   }   .
\end{align*}
Note that the numerator here can be rearranged:
\begin{align*}
    \exp \left(\averagei \ln 
                \frac{\vectorq_i^{(c)}}
                {1-\sum_{c'=1}^{k-1} \vectorq_i^{(c')}}\right) 
    &= \prod_{i=1}^m\left( 1 -\sum_{c'=1}^{k-1} \vectorq_i^{(c')}\right)^{-\frac{1}{m}}\prod_{i=1}^m \left(\vectorq_i^{(c)} \right)^{\frac{1}{m}},
\end{align*}
and the denominator can be written 
\begin{align*}
    1+ \sum_{c'=1}^{k-1} \exp \left( \averagei \ln \frac{\vectorq_m^{(c')}}{1-\sum_{c''=1}^{k-1} \vectorq_m^{(c'')}}\right)   
    &= 1 + \sum_{c'=1}^{k-1} \prod_{i=1}^m {\vectorq_i^{(c')}}^\frac{1}{m}\left(1 - \sum_{c''=1}^{k-1} \vectorq_i^{(c'')} \right)^{-\frac{1}{m}}\\
    &\hspace{-4cm} = 1 + \prod_{i'=1}^m\left(1 - \sum_{c''=1}^{k-1} \vectorq_{i'}^{(c'')} \right)^{-\frac{1}{m}}\sum_{c'=1}^{k-1} \prod_{i=1}^m {\vectorq_i^{(c')}}^\frac{1}{m}\\
    &\hspace{-4cm} = \prod_{i'=1}^m \left(1 - \sum_{c''=1}^{k-1} \vectorq_{i'}^{(c'')} \right)^{-\frac{1}{m}} \left( \prod_{i=1}^m \left(1 - \sum_{c''=1}^{k-1} \vectorq_i^{(c'')} \right)^{\frac{1}{m}} + \sum_{c'=1}^{k-1} \prod_{i=1}^m {\vectorq_i^{(c')}}^\frac{1}{m} \right).
\end{align*}
Putting the numerator and denominator back into the second expression of Equation~\eqref{eq:geo_mean} and using the definition of $\widehat{\vectorq}$, we find the first terms in both products cancel and we are left with the required result.

\paragraph{Full length $k$ Probability Vector:} If we do not use the minimally parameterized vectors, we would have the Bregman generator $\phi(\vectorp) = \sum_{c=1}^k \vectorp^{(c)} \ln \vectorp^{(c)}$. This gives the geometric mean, rather than the normalised version. To see this, we first note that
\begin{align*}
    \left(\nabla \phi (\vectorp)\right)^{(c)} &= 1 + \ln \vectorp^{(c)} = \boldeta^{(c)}\\ 
    \left(\left[\nabla \phi\right]^{-1} (\boldeta)\right)^{(c)} &= \exp \left(\boldeta^{(c)} -1 \right),
\end{align*}
and therefore the centroid combiner is
\begin{align*}
 \left([\nabla\phi]^{-1}\Big(\frac{1}{m}\sum_{m=1}^m \nabla\phi ({\bf q}_{m}) \Big)\right)^{(c)} &= \exp \left( \averagei 1 + \ln \vectorq_i^{(c)} - 1 \right) \\
  &= \exp \left( \averagei  \ln \vectorq_i^{(c)} \right) = \prod_{i=1}^m {\vectorq_i^{(c)}}^{\frac{1}{m}} .
\end{align*}

Note that this means that $\vectorqbar$ is not necessarily a valid probability vector. In fact, it is a valid probability vector \emph{only if} $\vectorq_1=\ldots=\vectorq_m$.

\newpage

\subsection{List of Bregman Centroids}
\label{appsubsec:bregmandiversity}
For quick reference purposes (see \citet{nielsen2009sided} for more details) we list centroids for different losses of interest to the ML community. The left Bregman centroid is defined:
    \begin{equation}
        \centroid{\vectorq} ~\defeq~ \argmin_{\bf z \in \mathcal{Y}} \mathbb{E}_D\Big[ \Bregman{{\mathbf{z}}}{\vectorq} \Big]
        ~=~
        \left[\nabla\phi\right]^{-1} \Big(\mathbb{E}_D \left[\nabla \generator{\vectorq} \right] \Big).
    \label{eq:centroid}
    \end{equation}

\noindent The centroid $\centroid{\vectorq}$ takes different forms dependent on the generator used.

\begin{table}[ht]
\centering
\def\arraystretch{2} 
\begin{tabular}{@{}llll@{}}
    \toprule[1.5pt]
    {\bf Loss}
        & \shortstack[l]{\bf Gradient\\ $\vectoreta=\nabla\phi({\bf  q})$} 
        & \bf \shortstack[l]{Inverse Grad.\\ ${\bf q}=\left[\nabla\phi\right]^{-1}(\vectoreta)$}  
        & \shortstack[l]{\bf Left Bregman Centroid\\$\centroid{\vectorq}:= \left[\nabla\phi\right]^{-1}\Big(\mathbb{E}_D \left[ \nabla \generator{{\bf  q}} \right] \Big)$} \\ 
    \midrule
    Squared 
        & $2 q$ 
        & $\frac{1}{2}\vectoreta$ 
        & $\ED [  q ]$ \\
    Itakura-Saito 
        & $-\frac{1}{ q}$ 
        & $-\frac{1}{\vectoreta}$ 
        & ${1}\big/{\Big(\ED\left[{1}/{ q}\right]\Big)}$\\
    Poisson loss 
        & $\ln q$
        & $\textrm{exp}(\vectoreta)$
        & $\textrm{exp}\big(\ED \left[ \ln q \right]\big)$ \\
    KL-divergence 
        & $\ln\frac{\mathbf{ q}^{(c)}}{1-\sum_{{c'}=1}^{k-1}\mathbf{ q}^{(c')}}$ 
        & $\frac{\textrm{exp}(\boldsymbol{\vectoreta}^{(c)})}{1+\sum_{c'=1}^{k-1}\textrm{exp}(\boldsymbol{\vectoreta}^{(c')})}$ 
        & $\frac{1}{Z}\textrm{exp}\Big(\ED \left[ \ln {\bf q} \right] \Big)$ \\[1ex]
    \bottomrule[1.5pt]
\end{tabular}
\caption{Common losses (see Table~\ref{tab:bregmangenerators}) and their Bregman centroids. In the case of KL-divergence, $Z$ is a normalizer to ensure a valid distribution.}
\label{tab:bregmancentroids}
\end{table}

%% file: 99_C_Section4proofs.tex
\newpage
\section{Proofs for Section \ref{sec:unified_diversity}}


\subsection{Proofs for Bias-Variance-Diversity and Effect Decompositions}\label{app:dependent}

\thegenbvd*

\begin{proof}
    Take the expected risk of $\vectorqbar$, and apply Proposition~\ref{prop:gen_ambig}, the generalized ambiguity decomposition:
    \begin{align}
        \EDb{\EXYb{\ell({\mathbf{Y}},{\vectorqbar})}} &= \EDb{\EXYb{\averagei \ell({\mathbf{Y}},{\vectorq_i})}} -  \EDb{\EXYb{\averagei \ell({\vectorqbar},{\vectorq_i})}}.\label{eq:gbvdd_intermediate_1}
    \end{align}
    Now apply Definition~\ref{def:gen_bv}, the generalised bias-variance decomposition, to the first term on the right:
    \begin{align}
        \mathbb{E}_D \Big[ \EXY \Big[ \averagei \ell({\bf Y},\vectorq_i) \Big] \Big] =& \nonumber\\
        &\hspace{-3.5cm}    
        \Ex\Bigg[\mathbb{E}_{{\bf Y}|{\bf X}} \left[ \ell({{\bf Y}},{\vectorYstar})\right]
        ~+~  \averagei \ell({\vectorYstar},{\centroid{\vectorq}_i})
        ~+~ \averagei\mathbb{E}_D\left[  \ell({\centroid{\vectorq}_i},{\vectorq_i})\right]\Bigg]. \label{eq:gbvdd_intermediate_2}
    \end{align}
Plugging Equation \eqref{eq:gbvdd_intermediate_2} into (\ref{eq:gbvdd_intermediate_1}) completes the proof.
\end{proof}

{\noindent \bf Further explanation on dependent/independent training schemes: } For the bias-variance decomposition of a single model $q$, we defined a random variable $D\sim P(\vectorx,y)^n$, i.e. over i.i.d. training sets of size $n$.  When we have an ensemble of models, it is common for each to have their own training set, e.g. the Bagging algorithm.
In this case, $D$ is redefined to a vector $[D_0, D_1, D_2, \dots, D_m]$, where $D_0\sim P(\vectorx,y)^n$ defines the overall training set supplied to the ensemble.
Each $D_i$, for $1..m$,  is a random variable defining an i.i.d. sample with replacement from the same fixed training set defined by $D_0$. Thus each $D_i$ is dependent on $D_0$, but not on any other $D_j$.  
However, it can also be that $D_t$ is dependent on $D_{t-1}$, for example as in Boosting algorithms.
In either case, if we consider any function of (without loss of generality) the first model, $q_1$, then by the law of total expectation, $\mathbb{E}_Y\Big[ \mathbb{E}_{X|Y} [X]\Big]=\mathbb{E}_X[X]$, we have:

\begin{equation}
\mathbb{E}_{D_{0\dots m}}\Big[ q_1(D_1) \Big] = \mathbb{E}_{{D_0,D_{2\dots m}}}\Big[ \mathbb{E}_{D_1|{D_0,D_{2\dots m}}}[q_1(D_1)] \Big]~=~\mathbb{E}_{D_1}\Big[q_1(D_1)\Big].    
\end{equation}
Thus, the function of the model is only dependent on the individual model's training set.


%
\diversityEffect*
\begin{proof}
Note that several terms on the right cancel, reducing to the left-hand side.
\end{proof}

\zeroOneDiversityDoesNotExist*
\begin{proof}
We show that 
$\averagei{\Lzeroone(y, q)} - \Lzeroone(y, \qbar)$
is necessarily dependent on the value of $y$ at a point $(\vectorx,y)$. Taking expectation
over $P(\vectorx, y)$ proves the result.
We show a case, where there is no ensemble combination rule such that the expression is independent of $y$.

\paragraph{The two-class case:} Define $\mathcal{Y}=\{1,2\}$.
For a fixed $\vectorx$, without loss of generality, let $p=0.6$ be the proportion of the $m$ models predicting class $1$. The combiner $\qbar$, provides a label $\in \bregmanDomain$.
%
 %
For both possible $\qbar$, the quantity $\averagei{\Lzeroone(y, q)} - \Lzeroone(y, \qbar)$ depends on $y$.
First, if we assume $\qbar=1$, we have:
 \begin{align*}
     \averagei{\Lzeroone(y, q_i)} - \Lzeroone(y, \qbar)  = 
     \begin{cases}
        0.4 - 0 = 0.4 & \textnormal{if }y=1\\
        0.6 - 1 = -0.4 & \textnormal{if }y=2
     \end{cases}
 \end{align*}
 Alternatively, when $\qbar=2$,
  \begin{align*}
     \averagei{\Lzeroone(y, q_i)} - \Lzeroone(y, \qbar)  = 
     \begin{cases}
        0.4 - 1 = -0.6 & \textnormal{if }y=1\\
        0.6 - 0 = 0.6 & \textnormal{if }y=2
     \end{cases}
 \end{align*}
 For both $\qbar$, the value of $\averagei{\Lzeroone(y, q_i)} - \Lzeroone(y, \qbar)$ is dependent on the true label $y$.
 \paragraph{The multiclass case:} Define $\mathcal{Y}=\{1,2, \dots, k\}$, and the proportion of the $m$ models predicting each class as $p_1, \dots, p_k$.  We
 set $p_1=0.6$, $p_2=0.4$, and zero for all other classes.
 From the two-class case we know that when $\qbar \in \{1,2\}$, there is a dependency on $y$.  This persists for $\qbar \in \{3, \ldots, k\}$, where we have
   \begin{align*}
     \averagei{\Lzeroone(y, q_i)} - \Lzeroone(y, \qbar)  = 
     \begin{cases}
        0.4 - 1 = -0.6 & \textnormal{if }y=1\\
        0.6 - 1 = -0.4 & \textnormal{if }y=2\\
        1.0 - 0 = 1 & \textnormal{if }y=\qbar\\
        1.0 - 1 = 0 & \textnormal{for all other classes}.
     \end{cases}
    \end{align*}
    For all $\qbar\in\mathcal{Y}$, 
     the expression $\averagei{\Lzeroone(y, q_i)} - \Lzeroone(y, \qbar)$ is dependent on the label $y$.
     
\end{proof}

\newpage

\subsection{Bias/Variance/Diversity-Effect Decomposition: Weighted Voting}\label{appsubsec:weightedDiversityEffect}

We now prove a bias-variance-diversity-effect decomposition for a weighted majority vote combination.
Given an ensemble of classification models $q_1, \ldots, q_m$
each outputting a label prediction from $\{1, \ldots, k\}$, 
and weights for those models $\alpha_1(D), \ldots, \alpha_m(D)$, we consider the ambiguity-effect decomposition for weighted plurality vote.
The weighted majority vote is
\begin{align*}
    \qbar(\vectorx; D) = \argmin_{c \in \{1, \ldots, k\}} \sum_{i=1}^m \frac{\alpha_i(D)}{\sum_{j=1}^m \alpha_j(D)} \Lzeroone(c, q_i),
\end{align*}
similarly, the centroid for an ensemble member is given by
\begin{align*}
    \centroid{q}_i = \argmin_{c \in \{1, \ldots, k\}} \EDb{ \frac{\alpha_i(D)}{\EDb{\alpha_i(D)}} \Lzeroone(c, q_i)},
\end{align*}
with ties broken randomly (the tie break procedure can be thought of as part of the random variable $D$, since $D$ implicitly contains all sources of stochasticity related to the model).

\begin{theorem}[Ambiguity-Effect Decomposition for Weighted Majority Vote]
With this, we can define a weighted effect decomposition as
\begin{align*}
    \Lzeroone(y, \qbar) = \underbrace{\sum_{i=1}^m  a_i(D) \Lzeroone(y, q_i)}_{\textnormal{weighted average loss}} - \left[\underbrace{\sum_{i=1}^m  a_i(D) \Lzeroone(y, q_i) - \Lzeroone(y, \qbar)}_{\textnormal{ambiguity-effect}}\right],
\end{align*}
where $a_i = \frac{\alpha_i}{\sum_{j=1}^m \alpha_j}$.
\end{theorem}

The validity of this theorem can be verified simply by cancelling terms on the right-hand side. The theorem tells us that the loss of an ensemble can be decomposed into a non-negative term (the weighted average loss of the ensemble members), and an ambiguity-effect term, which quantifies how much better (or worse) the performance of the ensemble is compared to the average member loss.
Using the same principle, we can also construct a bias/variance-effect decomposition which takes into account a weighting $\alpha(D)$.

\begin{theorem}[Bias/variance-effect Decomposition for Weighted Majority Vote]
\begin{align*}
    \EDb{ \alpha \Lzeroone(y, q)} = \EDb{\alpha}\Lzeroone(y, \centroid{q}) + \Big[\EDb{\alpha \Lzeroone(y, q)} - \EDb{\alpha}\Lzeroone(y, \centroid{q}) \Big] 
\end{align*}
\end{theorem}
Again, the proof of the result is immediate from considering which terms on the right cancel. However, it is worth considering how the decomposition works and what the terms mean.
Consider the decomposition when we replace $\alpha$ with a normalised version $\frac{\alpha}{\EDb{\alpha}}$, we get:
\begin{align*}
    \EDb{\frac{\alpha}{\EDb{\alpha}} \Lzeroone(y, q)} = \Lzeroone(y, \centroid{q}) + \left[\EDb{\frac{\alpha}{\EDb{\alpha}} \Lzeroone(y, q)} - \Lzeroone(y, \centroid{q}) \right].
\end{align*}

This is exactly bias-variance-effect decomposition that we have seen previously, but re-weighting the contributions of the different data sets. 
In fact, the two are equivalent, with the weights defining a new probability density function. Taking $P_D(\mathcal{D})$ as the probability density function over data sets, the decomposition above is exactly the bias-variance-effect decomposition with the new probability density function $Q_D(\mathcal{D}) = P_D(\mathcal{D})\frac{\alpha(\mathcal{D})}{\EDb{\alpha(D)}}$.
We can also easily reintroduce label noise and expose a noise term, and turning the bias into a bias-effect:

\begin{align*}
    \EYsb{\EDb{\frac{\alpha}{\EDb{\alpha}} \Lzeroone(Y, q)}} =& \underbrace{\EYsb{\Lzeroone(Y, Y^\ast)}}_{\textnormal{noise}} + \underbrace{\EYsb{\Lzeroone(Y, \centroid{q}) - \Lzeroone(Y, Y^\ast)}}_{\textnormal{weighted bias-effect}} \\
    &\underbrace{\left[ \EDb{\EYsb{\frac{\alpha}{\EDb{\alpha}} \Lzeroone(Y, f) - \Lzeroone(Y, \centroid{q})}} \right]}_{\textnormal{weighted variance-effect}}.
\end{align*}

We can now apply similar proof methodology, and get the following bias-variance-diversity effect decomposition for weighted majority vote.

\begin{proposition}[Bias-Variance-Diversity-Effect for Weighted Voting]
    Given $m$ classifiers $q_1, \ldots, q_m$, where the ensemble  is a weighted majority vote, i.e., $\qbar = \argmin_{z \in \bregmanDomain} \sum_{i=1}^m \alpha_{i} \Lzeroone(z, q_i)$ for weights $\alpha_1, \ldots, \alpha_m \in \mathbb{R}_+$, the ensemble loss admits the following decomposition, where the normalised weight is $a_i={\alpha_i}/{\sum_{j=1}^m \alpha_j}$.
    \begin{align*}
       \EYsb{\EDb{\Lzeroone(Y, \qbar)}}
       &= \underbrace{\mystrut{1.2em}\EYsb{\Lzeroone(Y, Y^\ast)}}_{\textnormal{noise}} ~+~ \underbrace{\sum_{i=1}^m \EDb{a_i ~ \EYsb{\Lzeroone(Y, \centroid{q}_i) - \Lzeroone(Y, Y^\ast)}}}_{\textnormal{weighted~average~bias-effect}}\\
       &\hspace{-3.4cm}+ \underbrace{\sum_{i=1}^m\EDb{a_i ~ \EYsb{\Lzeroone(Y, q_i) - \Lzeroone(Y, \centroid{q}_i)}}}_{\textnormal{weighted~average~variance-effect}}  - \underbrace{ \EDb{ \EYsb{{\sum_{i=1}^m a_i ~ \Lzeroone(Y, q_i) - \Lzeroone(Y, \qbar)}}}}_{\textnormal{diversity-effect}} \Bigg],
    \end{align*}
    where $\centroid{q}_i = \argmin_{z \in \bregmanDomain} \EDb{\alpha_i \Lzeroone(z, q_i)} $, noting that $\alpha_i$ and $q_i$ are both dependent on $D$.
\end{proposition}

As before, the veracity of this result can be seen by cancelling terms on the right-hand side.
AdaBoost produces a set of binary classifiers $h_i\in\{-1,+1\}$ and corresponding weights $\alpha_i(D)\in \mathbb{R}$, so setting $q_i=h_i$ allows immediate application of the decomposition. 
LogitBoost does not produce classifier/weight pairs, but instead a set of regression models each $g_i \in \mathbb{R}$. We can apply the decomposition by separating these into sign/magnitude components, giving a classification $q_i = \sign(g_i)$ and and weight: $\alpha_i(D) = |g_i|$.

%% file: 99_D_Section5proofs.tex
\newpage
\section{Proofs for Section~\ref{sec:bregman_diversity}}
\label{app:proofsSection5}


\subsection{Proof of Bregman Ambiguity / Bias-Variance-Diversity decompositions.}

Technically, the first two theorems in this sub-section are a corollary to the existence of a bias-variance decomposition for Bregman divergences \citep{pfau2013} combined with Theorem~\ref{the:gen_bvd}. However, for completeness and didactic purposes, we present separate proofs here. The first proof will use the Bregman three-point property \citep{banerjee2005clustering}:

\newcommand{\xone}{{\bf x}_1}
\newcommand{\xtwo}{{\bf x}_2}
\newcommand{\xthree}{{\bf x}_3}
\begin{definition}[Bregman three-point property]
    Given a convex set $\mathcal{S}\subseteq \mathbb{R}^k$, a strictly convex function $\phi:\bregmanDomain \rightarrow \mathbb{R}$, and any $\xone\in\bregmanDomain$ and $\xtwo,\xthree \in \textnormal{ri}(\mathcal{S})$, the following identity holds:
\begin{equation}
    \Bregman{\xone}{\xthree} = \Bregman{\xone}{\xtwo} + \Bregman{\xtwo}{\xthree} + \langle \xone-\xtwo ~,~ \nabla\phi(\xtwo)-\nabla\phi(\xthree) \rangle.
\end{equation}
\end{definition}

\BregmanAmbiguity*

\begin{proof}
\noindent We instantiate the three-point property with $\xone=\vectory$, $\xtwo=\vectorqbar$, and $\xthree=\vectorq_i$.
\renewcommand{\xone}{\vectory}
\renewcommand{\xtwo}{\vectorqbar}
\renewcommand{\xthree}{\vectorq_i}
\begin{equation}
    \Bregman{\xone}{\xthree} = \Bregman{\xone}{\xtwo} + \Bregman{\xtwo}{\xthree} + \langle \xone-\xtwo ~,~ \nabla\phi(\xtwo)-\nabla\phi(\xthree) \rangle
\end{equation}
Now average both sides of this expression over the $m$ models:
\begin{eqnarray}
    \averagei\Bregman{\xone}{\xthree} &=& \Bregman{\xone}{\xtwo} + \averagei\Bregman{\xtwo}{\xthree} + \averagei\Big[\langle \xone-\xtwo ~,~ \nabla\phi(\xtwo)-\nabla\phi(\xthree) \rangle\Big] \notag \\
    &=& \Bregman{\xone}{\xtwo} + \averagei\Bregman{\xtwo}{\xthree} + \langle \xone-\xtwo ~,~ \nabla\phi(\xtwo)-\averagei\nabla\phi(\xthree) \rangle \notag\\
    &=& \Bregman{\xone}{\xtwo} + \averagei\Bregman{\xtwo}{\xthree}, \label{app:eq:bregman_ambiguity_proof}
\end{eqnarray}
where in the final step we have rearranged the definition of $\vectorqbar$ to use that $\nabla\phi(\vectorqbar) = \averagei\nabla\phi(\vectorq_i)$.
Finally, rearranging the terms of \eqref{app:eq:bregman_ambiguity_proof} recovers the desired result, Equation~\eqref{eq:bregman_ambiguity}.
\end{proof}

\newpage

\bvdd*

\begin{proof}
    Take the expected risk of $\vectorqbar$, and apply the Bregman ambiguity decomposition:
    \begin{align}
        \EDb{\EXYb{\Bregman{\mathbf{Y}}{\vectorqbar}}} &= \EDb{\EXYb{\averagei \Bregman{\mathbf{Y}}{\vectorq_i}}} -  \EDb{\EXYb{\averagei \Bregman{\vectorqbar}{\vectorq_i}}}.\label{eq:bvdd_intermediate_1}
    \end{align}
    Now apply Pfau's decomposition, Equation~\eqref{eq:bregman_bv}, to the first term on the RHS, and we have
    \begin{align}
        \mathbb{E}_D \Big[ \EXY \Big[ \averagei B_{\phi}({\bf Y},\vectorq_i) \Big] \Big] =& \nonumber\\
        &\hspace{-3.5cm}    
        \Ex\Bigg[\mathbb{E}_{{\bf Y}|{\bf X}} \left[ \Bregman{{\bf Y}}{\vectorYstar}\right]
        ~+~  \averagei \Bregman{\vectorYstar}{\centroid{\vectorq}_i}
        ~+~ \averagei\mathbb{E}_D\left[  \Bregman{\centroid{\vectorq}_i}{\vectorq_i}\right]\Bigg]. \label{eq:bvdd_intermediate_2}
    \end{align}
Plugging Equation \eqref{eq:bvdd_intermediate_2} into (\ref{eq:bvdd_intermediate_1}) completes the proof.
\end{proof}

\ensbiasvariance*
\begin{proof}\\
Equation~\eqref{eq:ensemblebias} can be proven by applying Theorem~\ref{the:bregmanambiguity} to a set of centroid models, $\{\centroid{\vectorq}_i\}^m_{i=1}$. 
Equation~\eqref{eq:ensemblevariance} can be proven by applying~Theorem~\ref{the:bvdd} but substituting ${\bf y}=\centroid{\vectorqbar}$.
\end{proof}

\newpage

\subsection{Proofs for Section \ref{sec:furtherproperties}}
\label{app:furtherPropertyProofs}


{\bf Proposition 15}~{\em 
Assume a true probability vector, $\vectory\in \mathbb{R}^{k}$,
and a set of models $\{\vectorq_i\}_{i=1}^m$ combined by averaging, i.e., $\vectorq^\dagger=\averagei \vectorq_i$, then the cross-entropy loss is
\begin{align}
   \underbrace{\mystrut{2.em}-{ \vectory\cdot\ln {\vectorq^\dagger} }}_{\textnormal{ensemble cross-entropy}} ~=~ \underbrace{\mystrut{2.em}-\averagei \vectory\cdot\ln{\vectorq_i}}_{\textnormal{average cross-entropy}} ~-~ \underbrace{\mystrut{2.em}\sum_{c=1}^{k} {\vectory}^{(c)} \ln \dfrac{\frac{1}{m}\sum_{j=1}^{m} {\vectorq}_j^{(c)}} {\left(\prod_{i=1}^{m} {\vectorq}_i^{(c)}\right)^{\frac{1}{m}} }}_{\textnormal{ambiguity (label-dependent)}}, \label{app:eq:arithmetic_ambiguity}
\end{align}
where the second term is non-negative, thus the ensemble loss is guaranteed less than or equal to the average individual loss.\\
}

\begin{proof}
Take the average cross-entropy, and subtract the ensemble cross entropy:
\begin{align*}
       - \averagei \vectory \cdot \ln\vectorq_i - \Big[ -\vectory \cdot \ln \vectorq^\dagger\Big] &=  \sum_{c=1}^{k}\vectory^{(c)} \ln {\vectorq^\dagger}^{(c)} - \averagei \sum_{c=1}^{k} \vectory^{(c)}\ln \vectorq_i^{(c)} \\
       &=  \sum_{c=1}^{k}\vectory^{(c)} \ln {\vectorq^\dagger}^{(c)} - \sum_{c=1}^{k} \vectory^{(c)}\ln\Big( \prod_i \vectorq_i^{(c)} \Big)^{1/m} \\
       &=  \sum_{c=1}^{k}\vectory^{(c)}\ln \left( \frac{{\vectorq^\dagger}^{(c)}}{\prod_{i=1}^m\left( \vectorq_i^{(c)}\right)^{\frac{1}{m}}}\right)
\end{align*}
Using the definition of $\vectorq^\dagger$ and rearranging completes the derivation.
From the arithmetic-geometric mean inequality,
${\vectorq^\dagger}^{(c)}\geq \prod_{i=1}^m \left(\vectorq_i^{(c)}\right)^{{1}/{m}}$, implying that the term inside the logarithm is greater or equal to 1, and so the overall term is non-negative.
\end{proof}

\newcommand{\lastElem}[1]{\left( 1 - \sum_{c=1}^{k-1} (#1)^{(c)}\right)}


{\noindent \bf Proposition 16}~{\em Let $\vectorq^\dagger = \averagei \vectorq_i$, with $\vectorq_i\in [0,1]^k$. The expected cross-entropy admits the decomposition:
\begin{eqnarray}
    -\EDb{ {\vectory\cdot\ln{\vectorq^\dagger}  }} &=&\notag\\
        &&\hspace{-3.3cm} \underbrace{\mystrut{2.em}-\averagei\vectory\cdot\ln {\centroid{\vectorq}_i}}_{\textnormal{average bias}}
        \,+~\, \underbrace{\mystrut{2.em}\averagei\EDb{\KL{\centroid{\vectorq}_i}{\vectorq_i} }}_{\textnormal{average variance}} - \underbrace{\mystrut{2.em}\EDb{\sum_{c=1}^{k} {\vectory}^{(c)} \ln \dfrac{\frac{1}{m}\sum_{j=1}^{m} {\vectorq}_j^{(c)}} {\left(\prod_{i=1}^{m} {\vectorq}_i^{(c)}\right)^{\frac{1}{m}} } }}_{\textnormal{dependency}}.\notag
\end{eqnarray}}
\begin{proof}
    Starting with Equation~\eqref{app:eq:arithmetic_ambiguity} and taking the expectation over $D$, we have 
    \begin{align*}
    -\EDb{\vectory \cdot \ln \vectorq^\dagger} &= \EDb{-\averagei \vectory \cdot \ln \vectorq_i} - \EDb{\sum_{c=1}^{k} {\vectory}^{(c)} \ln \dfrac{\frac{1}{m}\sum_{j=1}^{m} {\vectorq}_j^{(c)}} {\left(\prod_{i=1}^{m} {\vectorq}_i^{(c)}\right)^{\frac{1}{m}} }}.
    \end{align*}
    Now, apply the KL bias-variance decomposition to the first term, and the result is proven.
\end{proof}